\documentclass[11pt,a4paper]{article} 
\usepackage{amsmath,amsthm,amssymb,graphicx,mathtools,tikz, algorithm, algorithmic, tcolorbox,setspace,mdframed,verbatim, dsfont,pifont}
\usepackage{natbib}
\setcitestyle{authoryear}
\usepackage{geometry}
\usepackage{bold-extra}
\usepackage{nicefrac}
\usepackage{pifont}
\usepackage[colorlinks=true, linkcolor=blue, urlcolor=blue,citecolor=blue]{hyperref}

\usetikzlibrary{positioning}

\newcommand{\R}{\mathbb{R}}
\newcommand{\E}{\mathbb{E}}

\newcommand{\N}{\mathcal{N}}
\newcommand{\norm}[1]{\|#1\|}
\newcommand{\vbrack}[1]{\langle #1\rangle}
\newcommand{\W}{\mathbf{W}}
\newcommand{\D}{\mathbf{D}}
\newcommand{\B}{\mathbf{B}}
\newcommand{\A}{\mathbf{A}}
\newcommand{\loss}{\textbf{loss}}

\newcommand{\h}{\mathbf{h}}
\newcommand{\w}{\mathbf{w}}

\newcommand{\id}{\mathbf{I}}
\newcommand{\x}{\mathbf{x}}
\newcommand{\losshat}{\widehat{\loss}}
\newcommand{\losstilde}{\widetilde{\loss}}

\newcommand{\back}{\text{Back}}
\renewcommand{\v}{\mathbf{v}}
\renewcommand{\u}{\mathbf{u}}
\renewcommand{\P}{\mathbb{P}}
\newcommand{\diag}{\text{diag}}

\newcommand{\Q}{\mathbf{Q}}
\newcommand{\qt}{\tilde{q}}
\newcommand{\zt}{\tilde{z}}
\newcommand{\kt}{\tilde{k}}
\newcommand{\tr}{\text{tr}}

\newcounter{main}
\numberwithin{main}{section}
\theoremstyle{definition}
\newtheorem{assumption}[]{Assumption}
\newtheorem{definition}[main]{Definition}
\newtheorem{theorem}[main]{Theorem}
\newtheorem{lemma}[main]{Lemma}
\newtheorem{corollary}[main]{Corollary}
\newtheorem{fact}[]{Fact}
\newtheorem*{remark}{Remark}
\theoremstyle{remark}

\numberwithin{equation}{section}

\begin{document}
 
 
 
\title{Convergence of End-to-End Training in Deep Unsupervised Contrasitive Learning}
\author{Zixin Wen\thanks{Dept. of Statistics, University of International Business and Economics, Beijing; \href{davidzxwen@icloud.com}{davidzxwen@icloud.com}.}}
\date{\today}

\maketitle
\begin{abstract}
    Unsupervised contrastive learning has gained increasing attention in the latest research and has proven to be a powerful method for learning representations from unlabeled data. However, little theoretical analysis was known for this framework. In this paper, we study the optimization of deep unsupervised contrastive learning. We prove that, by applying end-to-end training that simultaneously updates two deep over-parameterized neural networks, one can find an approximate stationary solution for the non-convex contrastive loss. This result is inherently different from the existing over-parameterized analysis in the supervised setting because, in contrast to learning a specific target function, unsupervised contrastive learning tries to encode the unlabeled data distribution into the neural networks, which generally has no optimal solution. Our analysis provides theoretical insights into the practical success of these unsupervised pretraining methods.
\end{abstract}

\section{Introduction}

Unsupervised representation learning has achieved enormous success in practical applications, especially in natural language processing, such as the famous word2vec \citep{mikolov2013efficient} and the groundbreaking advent of BERT \citep{devlin2019bert} and its variants as unsupervised pretrained language models. Among the unsupervised learning approaches, contrastive learning has gained increasing attention in the deep learning community. More surprisingly, as shown by \citet{he2019momentum}, unsupervised contrastively pretrained models can outperform their supervised counterparts in many downstream vision tasks, suggesting that the area of computer vision, which was previously dominated by supervised pretraining, can also benefit from unsupervised pretraining. Beyond these conventional approaches, unsupervised contrastive learning has also been employed in a variety of novel applications such as layer-wise representation learning \citep{lowe2019putting} and representation learning of the actual world \citep{kipf2019contrastive}. These studies together reflect the popularity and capability of the unsupervised contrastive methods.

In this paper, we view the unsupervised contrastive learning as a pretraining method, where the goal is to obtain pretrained representations that can be transferred to downstream tasks via fine-tuning. The benefit of doing unsupervised rather than supervised learning is its capability of leveraging the unlabeled data, which are more accessible and inexpensive relative to the labeled data. Developing and understanding unsupervised pretraining methods are necessary due to these limitations. 

However, besides the plentiful achievements in the practical side of deep unsupervised learning (and specifically, contrastive learning), recent theoretical studies focus mainly on supervised methods and their learning dynamics. Since the work of \citet{jacot2018neural,li2018learning}, the over-parameterization theory of deep learning has grown and brought about several breakthrough results on the convergence of deep neural networks trained by gradient descent or stochastic gradient descent, as shown in \citet{du2018gradient,allen2018convergence,zou2018stochastic, oymak2019towards, zou2019improved, ji2019polylogarithmic}. These analyses have contributed a lot to our understanding of the supervised deep learning. Nevertheless, the success of deep learning cannot be ascribed to supervised learning alone. It is unclear whether we can obtain similar results under the unsupervised setting, where there are no labels to fit or target functions to learn. This paper intends to fill this void by analyzing the optimization of unsupervised contrastive learning using deep neural networks in the over-parameterized regime.

In unsupervised contrastive learning, the networks learn through comparing examples by their feature representations. The main idea, as described in \citet{he2019momentum}, can be thought of as training encoders for a \textit{dictionary look-up} task. Consider a query \(q\) and a set of keys \(k_j\), where a query matches a key if they encode information of the same image (in vision) or they encode contextual messages coherent in a sentence (in NLP). At random initialization, the model is likely to match a query to a wrong key and incurs a large loss, and therefore needs to be trained to match the query to the right key. To formulate this idea mathematically, We consider the following loss function:
\begin{equation}\label{def:loss-intro}
    \mathcal{L} = - \E \Biggl[\log \frac{\exp(q^{\top}k_0)}{\sum_{j=0}^k\exp(q^{\top}k_{j})}\Biggr]
\end{equation}
where \(q = f^q(\x)\) is the query representation of \(\x\), \( k_0 = f^k(\x^{+})\) is the key representation of the positive example \(\x^{+}\), and \(k_j = f^k(\x_j)\) are the key representations of the negative examples \(\{\x_j\}_{j=1}^k\).  The encoders \(f^q\) and \(f^k\) are trained to capture the correlation between these examples and project them into a new feature space. 

Intuitively, minimizing the loss function \eqref{def:loss-intro} is similar to classify \(q\) as \(k_0\), which is a convex program. But in contrastive learning, both encoders \(f^q\) and \(f^k\) are updated at each iteration, which makes the contrastive loss \eqref{def:loss-intro} jointly non-convex for the outputs of two networks. This simultaneous updating scheme significantly complicates the analysis of its training dynamics, and motivates us to ask the following question: \textit{What solution can we obtain via unsupervised contrastive pretraining?} We answered this question in our paper and summarize our contribution as follows:

\begin{itemize}
    \item We show that, if the query and key encoders are sufficiently over-parameterized (the number of hidden nodes \(m\) is large enough), by applying end-to-end training that simultaneously updates the query and key encoders, one can find an approximate stationary solution for the non-convex contrastive loss in polynomial time.
\end{itemize}

\section{Related Work}

The result of this paper involves both the aspect of unsupervised contrastive learning and the guarantees for the optimization of deep learning. We discuss both sides below.

\subsection*{Unsupervised Contrastive Learning}

The first paper on contrastive learning is \citet{10.3115/1219840.1219884}, which contains almost all the important ideas for contrastive learning. \citet{hinton2006reducing} used the term contrastive loss for the first time, while their loss function is actually distance-based, similar to many other unsupervised methods. \citet{gutmann2010noise} and \citet{gutmann2012noise} proposed the noise contrastive estimation (NCE) which is widely-used today.

In natural language processing, many well-known unsupervised/self-supervised\footnote{We view self-supervised learning as a form of unsupervised learning, following \citet{he2019momentum}, as there is no formal difference in the existing literature. We use the term "unsupervised learning" as long as the learning procedure is "not supervised by human-annotated labels".} models can be thought of as certain forms of contrastive learning. \citet{mikolov2013efficient} proposed the revolutionary word2vec for contextual word embedding, which can be thought of as unsupervised contrastive learning using only one-layer query/key networks, and also they introduced the widely-used negative sampling \citep[see also][]{goldberg2014word2vec}. Some following work \cite{levy2014neural,li2015word,sharan2017orthogonalized,frandsen2018understanding} further characterized and developed word2vec via matrix/tensor decomposition. In the subsequent years many contextual embedding/language modelling methods have been proposed, say ELMo \citep{peters2018deep}, ULM-FiT \citep{howard2018universal}, BERT \citep{devlin2019bert} and its variants \citep{yang2019xlnet,lan2019albert}. The pretraining stage of these language models often involves inner products like \(f(\x)^{\top}\theta\) to match the context to the right words, which can be viewed as contrastive learning with deep query encoder and shallow key encoder. 

Besides language modeling, \citet{wu2018unsupervised} applied the NCE objective to perform unsupervised pretraining based on imageNet level data. \citet{oord2018representation} heuristically proved that contrastive learning maximizes the lower bound of the mutual information between the query and keys' representation. Further work such as \citet{hjelm2018learning,zhuang2019local,henaff2019data,tian2019contrastive} extended the applications of contrastive learning in computer vision. Very recently, the work of \citet{he2019momentum} and \citet{misra2019self} showed that models pretrained via unsupervised contrastive learning can outperform supervised pretrained counterparts in many downstream vision tasks. \citet{chen2020simple} showed that contrastive pretraining can achieve over \(76\%\) top-1 accuracy in imageNet classification by runing linear regression over frozen features.

On the theoretical side, \citet{ma2018noise} analyzed the statistical properties of the NCE objective and its effectiveness in natural language processing. \citet{arora2019theoretical} theoretically studied the generalization performance of unsupervised contrastive learning under the latent class framework proposed in their paper, which, as far as we know, is the first theoretical analysis of unsupervised pretraining. But their focus is on learning theory instead of optimization. 

\subsection*{Optimization of Deep Learning}

Previous to the emergence of over-parameterized analysis, much work has been done on the optimization of shallow neural networks, say \citet{tian2017analytical,zhong2017recovery,brutzkus2017globally,li2017convergence, du2017convolutional}. But most of the results in these papers are under stringent assumptions such as Gaussian distribution of input data or requiring special initialization methods (such as orthogonal initialization).

Recently there have been several breakthroughs in the optimization of deep neural networks in the over-parameterized regime. \citet{jacot2018neural} showed that as the width of the fully-connected network goes to infinity, the network converges to a feature map in the reproducing kernel Hilbert space induced by the Neural Tangent Kernel (NTK). \citet{li2018learning} independently proved the convergence of stochastic gradient descent for over-parameterized two-layer networks. Following these two papers, \citet{du2018gradient,allen2018convergence,zou2019improved} proved the convergence of (stochastic) gradient descent to a global minimum for deep neural networks (fully-connected, CNN and ResNet) if they are sufficiently over-parameterized. Follow-up work \citep{wu2019global,oymak2019towards,zou2019improved,ji2019polylogarithmic,chen2019much} further improved the convergence rates and over-parameterization conditions under different assumptions and settings. However, none of the existing papers have ever touch the setting of unsupervised deep learning, which is the focus of the current paper.

\section{Preliminaries}

\subsection{Notations}

We denote \([n] = \{1,\dots,n\}\), and \(S = \{\x_i\}_{i=1}^n\) to be our training set, \(S^{\setminus i} = S \setminus \{\x_i\}\) as the training set without the data point \(\x_i\). We use \(\N(0,\id_m)\) to denote the multivariate standard Gaussian distribution with \(m\)-dimensions. For a vector \(\v = (v_1,\dots,v_m)^{\top} \in \R^m\), we denote \(\|v\|_2 = (\sum_{i=1}^m v_i^2)^{1/2}\) to be its \(\ell_2\) norm. For a matrix \(\A = (a_{i,j})_{m\times n}\) we denote \(\|\A\|_0\) to be the number of non-zero entries of \(\A\), \(\|\A\|_2\) to be its spectral norm. For two matrices \(\A=(a_{ij})_{m\times n},\, \B = (b_{ij})_{m \times n}\), we denote \(\vbrack{\A,\B} = \vbrack{\A,\B}_F = \tr(\A^{\top}\B) = (\sum_{i,j}a_{ij}b_{ij})^{1/2}\) to be its trace inner product and \(\|\A\|_F = \sqrt{\vbrack{\A,\A}}\) to be the Frobenius norm of \(\A\). For neural network parameters \(\W = (\W_0,\dots,\W_L)\) and \(\W' = (\W_0,\dots,\W_L)\in \mathcal{W}\), where \(\mathcal{W} := \R^{m\times \mathfrak{b}}\times \R^{(m\times m)\cdot(L-1)}\times \R^{d\times m}\), we let 
\(\vbrack{\W,\W'}:= \sum_{l=0}^L \vbrack{\W_l,\W'_l}\) and \(\|\W\|_F = \sqrt{\vbrack{\W,\W}}\). We use \(O(\cdot),\Omega(\cdot)\) and \(\Theta(\cdot)\) to denote the standard big-O, big-Omega and big-Theta notations, only hiding positive constants.
    
\subsection{Problem Setup}
    
The method of contrastive learning involves two neural networks, and we define their architectures in the definition below.
    
\begin{definition}[Network Architecture]
    In contrastive learning, we need two neural networks, the query encoder \(f^q_{\W}\) and the key encoder \(f^k_{\theta}\), and without loss of generality we let them to be \((L+1)\)-layer fully connected networks with the same architecture. Our definitions of \(f^q_{\W}\) and \(f^k_{\theta}\) are:
    \begin{align*}
        f^q_{\W}(\x) = \W_{L}\sigma(\cdots\sigma(\W_1\sigma(\W_0\x))),\qquad f^k_{\theta}(\x) = \theta_{L}\sigma(\cdots\sigma(\theta_1\sigma(\theta_0\x)))
    \end{align*}
    where \(\sigma(\cdot)\) is the ReLU activation. \(\W_L,\theta_L \in \R^{d\times m}\), \(\theta_l, \W_l \in \R^{m\times m}\) for every \(1 \leq l \leq L-1\), \(\W_0, \theta_0 \in \R^{m\times \mathfrak{b}}\), where \(\mathfrak{b}\) is the input dimension, \(d\) is the output dimension. We use the compact notation \(\W = (\W_l)_{l=0}^L\) and \(\theta = (\theta_l)_{l=0}^L\) to denote the parameters of the two networks.
\end{definition}

\begin{remark}
    In practice, the architectures of query and key encoders are possibly different. We adopt the setting where they are of the same architecture, which is not essential and can be modified to the more general setting. However, such a modification may slightly complicate the final result and we decide not to carry it out.
\end{remark}

We present our initialization scheme of the network parameters below, which is knwon as He initialization \cite{he2015delving}, and has been adopted in the theoretical work \citet{li2018learning,allen2019learning,zou2018stochastic,zou2019improved}.
    
\begin{definition}[Initialization]\label{def:initialization}
    The initializations of our parameters \(\W,\theta\) are defined as follows,
    \begin{itemize}
        \item \((\W_0)_{i,j},(\theta_0)_{i,j} \stackrel{\text{i.i.d.}}\sim \N\bigl(0,\frac{2}{m}\bigr)\)  for \((i,j)\in [m]\times [\mathfrak{b}]\);
        \item \((\W_l)_{i,j},(\theta_l)_{i,j} \stackrel{\text{i.i.d.}}\sim \N\bigl(0,\frac{2}{m}\bigr)\)  for \((i,j)\in [m]\times [m]\) and every \(l \in [L-1]\);
        \item \((\W_L)_{i,j},(\theta_L)_{i,j} \stackrel{\text{i.i.d.}}\sim \N\bigl(0,\frac{1}{d}\bigr)\)  for \((i,j)\in [d]\times [m]\).
    \end{itemize}
\end{definition}
    
We present our definition of the contrastive loss function below, which lies in the core of this paper.
    
\begin{definition}[Contrastive Loss]\label{def:loss-function}
    Fixed \(k\) as the number of negative samples. For a specific sample \(\x_i \in S\), we select \(\x_{i,1},\dots,\x_{i,k} \in S^{\setminus i}\) to be its negative samples. Using our query encoder \(f^q_{\W}\) and key encoder \(f^k_{\theta}\), we represent these data points as query \(q_i = f^q_{\W}(\x_i)\) and keys \(k_{i,0}=f^k_{\theta}(\x_i)\), \(\ k_{i,j} = f^k_{\theta}(\x_{i,j})\). The contrastive loss of \(\x_i\) to negative samples \(\{\x_{i,j}\}_{j=1}^k\) is defined as 
    \begin{align}\label{def:loss-function-specific}
        \ell(f^q_{\W},f^k_{\theta},\x_i,\{\x_{i,j}\}_{j=1}^k) = - \log \Biggl[\frac{\exp( q_i^{\top} k_{i,0} )}{\sum_{j=0}^k \exp(q_i^{\top} k_{i,j})}\Biggr]
    \end{align}
    which intuitively can be viewed as \((k+1)\)-way classification loss that tries to classify \(q_i\) as \(k_{i,0}\). We minimize the following total loss
    \begin{align}\label{def:loss-function-total}
        L_S(\W, \theta) = \frac{1}{n}\sum_{i=1}^n \E^{Neg(i)}\bigl[ \ell(f^q_{\W},f^k_{\theta}, \x_i, \{\x_{i,j}\}_{j=1}^k)\bigr]
    \end{align}
    where \(\E^{Neg(i)}\) is defined as the expectation over the uniform sampling of all negative samples \(\{\x_{i,j}\}_{j=1}^k \subset S^{\setminus i}\).
\end{definition}  

\begin{remark}
    This form of contrastive loss is designed for the pretext task \textit{instance-level discrimination} \citep{wu2018unsupervised}, which treats each image as a distinct class of its own. The resulting negative sampling procedure can be described as \textit{one-against-all} negative sampling. Similar contrastive loss functions are also used in practical work \citet{he2019momentum} and \citet{chen2020simple}.
\end{remark}

We present the algorithm of end-to-end contrastive learning via gradient descent below. This algorithm is described in Figure 2 of \citet{he2019momentum} as an alternative approach for MoCo, and is implemented in section 4.1 in \citet{he2019momentum}, where they showed that it is almost as equally competitive as MoCo. The analysis of better algorithms such as MoCo requires dealing with more practical issues that are hard to analyze mathematically.
    
\begin{algorithm}[H]
\caption{End-to-end training via gradient descent}\label{alg:GD}
\begin{algorithmic}
    \STATE\textbf{input:} Training data \(S = \{\x_i\}_{i=1}^n\), step sizes \(\eta, \gamma\), total number of iterations \(T\).
    \STATE\textbf{initialization:} Initialize \(\W^{(0)},\theta^{(0)}\) randomly, following \hyperref[def:initialization]{Definition~\ref*{def:initialization}}.
    \FOR{\(t = 0,\dots,T-1\)}
    \STATE \(\W^{(t+1)} = \W^{(t)} - \eta \nabla_{\W} L_S(\W^{(t)},\theta^{(t)})\)
    \STATE \(\theta^{(t+1)} = \theta^{(t)} - \gamma \nabla_{\theta} L_S(\W^{(t)},\theta^{(t)})\)
    \ENDFOR 
    \STATE \textbf{output:}{\(\{\W^{(t)},\theta^{(t)}\}_{t=0}^T\)}
\end{algorithmic}
\end{algorithm}

\begin{remark}
    In practical papers such as \citet{he2019momentum,tian2019contrastive}, they usually optimize the networks by performing stochastic gradient descent with respect to a minibatch of data and a random set of negative examples. In practice the adoption of this doubly stochastic algorithm is due to the limitations of computation resources. In our analysis we instead evaluate the contrastive loss against all possible negative examples and perform gradient descent with repect to this non-random loss, which makes the algorithm non-random. The analysis of stochastic algorithm would significantly complicate the analysis. And we remark that the state-of-the-art analysis for stochastic gradient descent with respect to cross-entropy (logistic) loss for neural networks \citep[see][]{ji2019polylogarithmic,chen2019much} usually assume that there exist a "stochastic oracle", which is not applicable to our setting.
\end{remark}

\subsection{Assumptions}

The first assumption we made is that all the data points lie in the \(1\)-sphere with respect to the \(\norm{\cdot}_2\) norm. 
    
\begin{assumption}[Normalization]\label{assumption:normalized-data}
    Every training data point \(\x_i \in S\) satisfies \(\norm{\x_i}_2 = 1\).
\end{assumption}
    
This assumption is common in deep learning theory literature. As existing papers \cite{du2018gradient,allen2019learning,cao2019generalization} have pointed out, restricting the inputs \(\x_i\) to the \(1\)-sphere is not essential, and can be relaxed to requiring \(c_1 \leq \norm{\x_i}\leq c_2\) for some absolute constants \(c_2 > c_1 > 0\). 
    
Our second assumption is the non-degeneracy of data points, which first appeared in the papers \citet{li2018learning}, and has been adopted and further modified by \citet{allen2018convergence,zou2018stochastic,oymak2019towards,zou2019improved,chen2019much}.
\begin{assumption}[Non-degeneracy]\label{assumption:separateness}
    There exist a universal constant \(\delta > 0\) such that, for any \(i,j \in [n]\) with \(i\neq j\),
    \begin{displaymath}
        \norm{\x_i - \x_j}_2 \geq \delta
    \end{displaymath}
\end{assumption}
    
\begin{remark}
    In \citet{du2018gradient}, they have shown that the above data non-degeneracy assumption can implies \(\lambda_{\min}(\mathbf{K}^{(L)}) > 0\), where \(\mathbf{K}^{(L)}\) is the Gram matrix, which is also known as the Neural Tangent Kernel \citep{jacot2018neural} (see their papers for details).
\end{remark}

We also remark that the assumptions on the data in this paper are no more than existing papers studying the supervised setting. It is interesting whether the result in this paper would still hold if we the second assumption is significantly weakened by only requiring separation between groups of data points as in \citet{chen2019much}.

\section{Main Theory}

Before presenting our convergence theorem, we give a necessary definition.

\begin{definition}[\(\loss\)-vectors]\label{def:loss-vector-main}
    Denote \(q_i = f^q_{\W}(\x_i)\) and \(k_i = f^k_{\theta}(\x_i)\), we define
    \begin{align*}
        \losstilde = (\losstilde_i)_{i=1}^n = \bigl(\partial L_S/ \partial q_i\bigr)_{i \in [n]}, \qquad \losshat = (\losshat_i)_{i=1}^n = \bigl(\partial L_S/\partial k_i\bigr)_{i \in [n]}
    \end{align*}
    And we further define \(\loss = (\losstilde,\losshat)\) as our surrogate objective.
\end{definition}

Now We present our main theorem for end-to-end contrastive learning.

\hypertarget{thm}{}
\begin{theorem}\label{thm:convergence} 
    For any \(\varepsilon \in (0,1)\), \(\delta \leq O(1/L)\), suppose \hyperref[assumption:normalized-data]{Assumption~\ref*{assumption:normalized-data}} and \hyperref[assumption:separateness]{Assumption~\ref*{assumption:separateness}} holds, with over-parameterization condition
    \begin{align*}
        m &\geq \Omega\bigl(n^{15}L^{12}(\log m)^5\delta^{-5}\varepsilon^{-2}\bigr) \quad d\geq\Omega(\log^2 m)
    \end{align*}
    and if we perform \hyperref[alg:GD]{Algorithm~\ref*{alg:GD}}, with step sizes
    \begin{displaymath}
        \eta,\gamma = \Theta( d\varepsilon^2\delta^2/(n^{7}L^2km))
    \end{displaymath}
    then with probability at least \(1 - O(m^{-1})\) over the initialization, we have
    \begin{displaymath}
        \frac{1}{T}\sum_{t=0}^{T-1}\norm{\loss^{(t)}}_2 \leq \varepsilon \quad \text{for }T = \Theta\biggl(\frac{n^{10}L^{2}k}{\delta^3}\cdot \frac{1}{\varepsilon^4}\biggr) 
    \end{displaymath}
    where the \(\loss^{(t)}\)-vectors are defined in \hyperref[def:loss-vector-main]{Definition \ref*{def:loss-vector-main}}, with \(f^q\) and \(f^k\) parameterized by \(\W^{(t)}\) and \(\theta^{(t)}\) respectively.
\end{theorem}

As  mentioned in \citet{allen2018convergence}, the result of finding weight matrices \(\W\) that satisfies \(\|\nabla_f L_S\| \leq \varepsilon\) cannot be derived from the classical theory of finding approximate saddle points for non-convex objectives. And since in our end-to-end training we update two neural networks simultaneously, the interaction between these two networks during the optimization process makes it even harder for the optimization analysis.

Note that in contrast to existing work on the convergence of supervised training, we require the output dimension \(d\) to be sufficiently large (of magnitude \(\Theta(\log^2 m)\)). This requirement is necessary for both query encoder and key encoder to project sufficient information onto the output space and contrast between each queries \(q_i = f^q_{\W}(\x_i)\) and keys \(k_{j} = f^k_{\theta}(\x_j)\). Without this requirement, it would be difficult for the outputs to represent the high-dimensional information learned by the over-parameterized hidden layers. And also this requirement of \(d\) is not impractical because it is only for pretraining. One can always add a new fully-connected layer on top in the fine-tuning stage.

Our proof of \hyperref[thm:convergence]{Theorem \ref*{thm:convergence}} relies on two technical lemmas, and we shall elaborate them below. 

\subsection{Main Technical Lemmas}

We present two lemmas below that are the key components of our final convergence proof. The first lemma concerns the gradient bounds for updating both \(\W\) and \(\theta\). The proof of \hyperref[lem:gradient-bounds-main]{Lemma \ref*{lem:gradient-bounds-main}} is in \hyperref[sec:gradient-bounds]{Appendix~\ref*{sec:gradient-bounds}}.

\begin{lemma}[Gradient Bounds]\label{lem:gradient-bounds-main}
    Suppose \(\omega, \tau \leq O(\delta^{3/2}/(n^{3}L^6\log^{3/2} m))\) then with probability at least \(1 - e^{-\Omega(m\omega^{3/2}L)} - e^{-\Omega(m\tau^{3/2}L)}\) over the randomness of initialization, if \(\W \in B(\W^{(0)},\omega)\) and \(\theta\in B(\theta^{(0)},\tau)\), the following holds.
    \begin{itemize}
        \item For \(\norm{\nabla_\W L_S(\W,\theta)}_F\), we have 
        \begin{align*}
            \bigl\| \nabla_{\W}L_S(\W,\theta)\bigr\|_F^2 &\geq \Omega\biggl(\frac{m\delta}{n^3d}\biggr)\sum_{i=1}^n\|\losstilde_i\|_2^2 \\
            \bigl\| \nabla_{\W}L_S(\W,\theta)\bigr\|_F^2 &\leq O\biggl(\frac{Lm}{nd}\biggr)\sum_{i=1}^n\|\losstilde_i\|_2^2
        \end{align*}
        \item For \(\norm{\nabla_{\theta}L_S(\W,\theta)}_F\), we have
        \begin{align*}
            \bigl\|\nabla_{\theta} L_S(\W,\theta)\bigr\|_F^2 &\geq \Omega\biggl(\frac{m\delta}{n^3 d}\biggr)\sum_{i=1}^n\|\losshat_i\|_2^2 \\
             \bigl\|\nabla_{\theta} L_S(\W,\theta)\bigr\|_F^2 &\leq O\biggl(\frac{Lm}{nd}\biggr) \sum_{i=1}^n\|\losshat_i\|_2^2
        \end{align*}
    \end{itemize}
    where the \(\losstilde\) and \(\losshat\)-vectors are defined in \hyperref[def:loss-vector-main]{Definition~\ref*{def:loss-vector-main}}.
\end{lemma}

The second lemma verifies the semi-smoothness properties for updating both the query encoder \(f^q_{\W}\) and the key encoder \(f^k_{\theta}\) simultaneously. The semi-smoothness condition instead of Lipschitz smoothness is due to the non-smooth property of ReLU activations, as illustrated in \citet{allen2018convergence}. Our derivations of the semi-smoothness lemma is different in many aspects to the original one in \citet{allen2018convergence}, since not only do we need to simultaneously update two neural networks, we also need to compute the exact form of the gradient of loss function to the outputs of these two neural networks \(\losstilde =(\partial L_S/\partial q_i)_{i=1}^n\) and \(\losshat = (\partial L_S/\partial k_i)_{i=1}^n\) which is complicated after taking expectations with respect to negative sampling.

\begin{lemma}[The Semi-smoothness Properties]\label{lem:semi-smoothness-main}
    For any perturbations \(\norm{\W'}_2\leq \omega\) and \(\norm{\theta'}_2\leq \tau\), where
    \begin{displaymath}
        \omega, \tau\in [\Omega(\sqrt{d/m}),\ O(1/ (L^{9/2} (\log m)^{3/2}) )]
    \end{displaymath}
    and \(\W \in B(\W^{(0)},\omega),\theta \in B(\theta^{(0)},\tau)\) such that 
    \begin{displaymath}
        \W + \W' \in B(\W^{(0)},\omega),\quad \theta + \theta' \in B(\theta^{(0)},\tau)
    \end{displaymath}
    we have, with probability at least \(1 - \exp({-\Omega(m\omega^{3/2}L)})-\exp({-\Omega(m\tau^{2/3}L)})\) over the randomness of initialization, the following inequality holds,
    \begin{align}\label{ineq:semi-smoothness}
        \begin{split}
        L_S(\W + \W',\theta + \theta')\leq\ &  L_S(\W,\theta) + \vbrack{\nabla_{(\W,\theta)} L_S(\W,\theta),(\W', \theta')}\\
        & + O\biggl(\frac{\omega^{1/3}L^2\sqrt{m\log m}}{\sqrt{nd}}\biggr) \norm{\losstilde}_2\cdot\norm{\W'}_F \\
        &+ O\biggl(\frac{\tau^{1/3}L^2\sqrt{m\log m}}{\sqrt{nd}}\biggr)\norm{\losshat}_2\cdot\norm{\theta'}_F\\
        & + O\biggl(\frac{kL^2m^2}{d^2}\biggr)\Bigl(\tau^2\norm{\W'}_F^2+\omega^2\norm{\theta'}_F^2\Bigr)
        \end{split}
    \end{align}
\end{lemma}

\section{Proof Techniques}\label{sec:proof-techniques}

\subsection{Key Facts}\label{subsec:key-facts}

Since the contrastive loss function defined in \hyperref[def:loss-function]{Definition \ref*{def:loss-function}} is inherently different in form to the loss functions used in supervised learning, and also since we have taken expectation with respect to the negative sampling, we need to derive some basic facts of how the gradient is calculated for both the query and key encoders. The exact calculations are done in \hyperref[subsec:key-calculations]{Appendix~\ref*{subsec:key-calculations}}.

For notational convenience in the expositions below, we denote
\begin{align} \label{notation:query-key}
    \begin{split}
    q_i = f^q_{\W}(\x_i),\quad z_{j} = f^k_{\theta}(\x_{j}) - f^k_{\theta}(\x_i), \quad z_{i,j} = f^k_{\theta}(\x_{i,j}) - f^k_{\theta}(\x_i)    
    \end{split}
\end{align}
 
The form of \(\losstilde\)-vectors are directly to compute from the our definition of contrastive loss.

\begin{fact}[\(\losstilde\)-vector]\label{fact:loss-vec}
    For each \(i \in [n]\), the \(\losstilde_i\)-vector is the following vector obtained from calculating the gradient of \(L_S(\W,\theta)\) with respect to the query encoder \(q_i = f^q_{\W}(\x_i)\):
    \begin{align*}
        \losstilde_i &= \partial L_S/\partial q_i = \E^{Neg(i)}\biggl[\sum_{j=1}^k\frac{\exp(q_i^{\top}z_{i,j})}{1 + \sum_{s=1}^j\exp(q_i^{\top}z_{i,j})}\cdot z_{i,j}\biggr]
    \end{align*}
    where the expectation \(\E^{Neg(i)}\) is taken with respect to the uniform sampling of \(\{\x_{i,j}\}_{j=1}^k \subset S^{\setminus i}\). 
\end{fact}

The exact form of \(\losshat\)-vectors are more subtle, and we present it below.

\begin{fact}[\(\losshat\)-vector]\label{fact:losshat-vec}
    For each pair \((i,j) \in [n]\times[n]\) such that \(i\neq j\), we denote the \(\losshat(\x_i,\x_j)\)-vector to be the following vector :
    \begin{align*}
        &\losshat(\x_i,\x_j) :=  \frac{1}{\binom{n-1}{k}}\sum_{\x_j \in \{\x_{i,s}\}_{s\in [k]} \subset S^{\setminus i} } \frac{\exp(q_i^{\top}z_{j})}{1 + \sum_{s=1}^k \exp(q_i^{\top} z_{i,s})}\cdot q_i
    \end{align*}
    where \(q_i,z_j,z_{i,s}\) is defined in \eqref{notation:query-key}. The summation is over all set of negative samples \(\{\x_{i,s}\}_{s=1}^k \subset S^{\setminus i}\) that contains the sample \(\x_j\). Now the \(\losshat_i\)-vector can be calculated as
    \begin{align*}
        \losshat_i &= \partial L_S/\partial k_i = \sum_{j\neq i}(\loss(\x_j,\x_i) - \loss(\x_i,\x_j))
    \end{align*}
\end{fact}

\subsection{Proof Overview of Technical Lemmas}

We outline the proof of \hyperref[lem:gradient-bounds-main]{Lemma \ref*{lem:gradient-bounds-main}} and \hyperref[lem:semi-smoothness-main]{Lemma \ref*{lem:semi-smoothness-main}} here. Firstly we define the following notations:
\begin{align*}
    \h^q_{i,0} = \sigma(\W_0 \x_i)\quad \h^q_{i,l} = \sigma(\W_l \h^q_{i,l-1}) \quad \h^k_{i,0} = \sigma(\theta_0 \x_i)\quad \h^k_{i,l} = \sigma(\theta_l \h^q_{i,l-1})
\end{align*}
and also diagonal matrices \(\D^q_{i,l}\) and \(\D^k_{i,l}\) as
\begin{align*}
    (\D^q_{i,l})_{a,a}:= \mathds{1}\{\sigma(\W_l \h^q_{i,l-1})_{a}\geq 0\} \qquad (\D^k_{i,l})_{a,a}&:= \mathds{1}\{\sigma(\theta_l \h^k_{i,l-1})_{a}\geq 0\}
\end{align*}

\paragraph{Gradient bounds:}

For notational convenience, we define the back-propagation matrices \(\back^q_{i,l}\) and \(\back^k_{i,l}\) as
\begin{align*}
    &\back^q_{i,l} = \W_L \D_{i,L-1}\cdots \D_{i,l}\W_l, \qquad \back^k_{i,l} = \theta_L \D_{i,L-1}\cdots \D_{i,l}\theta_l
\end{align*}

From the derivation of \hyperref[fact:loss-vec]{Fact \ref*{fact:loss-vec}} and \hyperref[fact:losshat-vec]{Fact \ref*{fact:losshat-vec}} we can transform the gradient \(\nabla_{\W_l}L_S(\W,\theta)\) and \(\nabla_{\theta_l}L_S(\W,\theta)\) into more operable forms 
\begin{align*}
    &\nabla_{\W_l}L_S(\W,\theta) = \frac{1}{n}\sum_{i=1}^n \D_{i,l}^q\Bigl((\back_{i,l+1}^q)^{\top}\losstilde_i\Bigr)\h_{i,l-1}^{q,\top} \\
    &\nabla_{\theta_l} L_S(\W,\theta) = \frac{1}{n}\sum_{i=1}^n  \D_{i,l}^k\Bigl((\back_{i,l+1}^k)^{\top} \losshat_i \Bigr) \h_{i,l-1}^{k,\top}
\end{align*}
From the initialization, the norm of the product \((\back_{i,l+1}^q)^{\top}\losstilde_i\) is of magnitude \(\sim \sqrt{m/d}\norm{\losstilde_i}_2\) (and similarly for \((\back_{i,l+1}^k)^{\top} \losshat_i\)). The lower bounds can be derived from the randomness decomposition arguement in \citet{allen2018convergence} and an improved version in \citet{zou2019improved}. The upper bounds follows from the naive bounds \(\|(\back_{i,l+1}^q)^{\top}\losstilde_i\|_2 \leq O(\sqrt{m/d})\norm{\losstilde_i}_2\) with high probability.

\paragraph{Semi-smoothness:}

To derive the semi-smoothness for updating two neural networks, we start from the function \(\ell(f^q_{\W},f^k_{\theta},\x_i,\{\x_{i,j}\}_{j=1}^k)\) defined in \hyperref[def:loss-function]{Definition \ref*{def:loss-function}}. We transform it to
\begin{align*}
    &\ell(f^q_{\W},f^k_{\theta},\x_i,\{\x_{i,j}\}_{j=1}^k) =  \log \biggl[ 1 + \sum_{j=1}^k \exp(q_i^{\top}z_{i,j}) \biggr]
\end{align*}
where \(q_i,z_{i,j}\) is defined in \eqref{notation:query-key}. Clearly this function is convex with respect to \(q_i^{\top}z_{i,j}\), and from simple calculation we showed that this function is \(1\)-Lipschitz smooth. Thus we obtain a second order bound with respect to \(q_i^{\top}z_{i,j}\)
\begin{align*}
    &\ell(f^q_{\widetilde{\W}},f^k_{\widetilde{\theta}},\x_i,\{\x_{i,j}\}_{j=1}^k) \leq \ell(f^q_{\W},f^k_{\theta},\x_i,\{\x_{i,j}\}_{j=1}^k) \\
    & \qquad + \underbrace{\sum_{j = 1}^k \frac{\exp(q_i^{\top}z_{i,j} )}{1 + \sum_{s=1}^k \exp(q_i^{\top}z_{i,s})}\cdot \bigl( \qt_i^{\top}\zt_{i,j} - q_i^{\top}z_{i,j} \bigr)}_{\text{\ding{172}}} + \underbrace{\frac{1}{2} \sum_{j=1}^k (\tilde{q}_i^{\top}\tilde{z}_{i,j} - q_i^{\top}z_{i,j})^2}_{\text{\ding{173}}}
\end{align*}
where \(\tilde{q}_i, \tilde{z}_{i,j}\) are the \(q_i,z_{i,j}\) paramterized by \(\widetilde{\W} = \W + \W'\) and \(\widetilde{\theta}=\theta + \theta'\), which are not far from the initialization. We decompose \(\tilde{q}_i^{\top}\tilde{z}_{i,j} - q_i^{\top}z_{i,j}\) into three terms
\begin{displaymath}
    \underbrace{(\qt_i - q_i)^{\top}z_{i,j}}_{\text{\ding{174}}} +  \underbrace{q_i^{\top}(\zt_{i,j} - z_{i,j})}_{\text{\ding{175}}} + \underbrace{(\qt_i - q_i)^{\top}(\zt_{i,j} - z_{i,j})}_{\text{\ding{176}}}
\end{displaymath}
and tackle them separately. For the terms \ding{172} \& \ding{174} and , after taking expectation with respect to negative sampling, we obtain
\begin{displaymath}
    \frac{1}{n}\sum_{i=1}^n \vbrack{\losstilde_i,\qt_i - q_i}
\end{displaymath}
and similarly, for the term \ding{172} \& \ding{175}, we can take expectation with respect to negative sampling and rearrange to get 
\begin{displaymath}
    \frac{1}{n}\sum_{i=1}^n \vbrack{\losshat_i,\kt_i - k_i}
\end{displaymath}
where \(\kt_i = f^k_{\widetilde{\theta}}(\x_i)\) and \(k_i = f^k_{\theta}(\x_i)\). The terms \ding{172} \& \ding{176} and \ding{173} \& (\ding{174} \(+\) \ding{175} \(+\) \ding{176}) can be bounded as 
\begin{displaymath}
    O(kL^2m^2/d^2)\cdot (\tau^2\|\widetilde{\W}-\W\|_2^2 + \omega^2\|\widetilde{\theta}-\theta\|_2^2)
\end{displaymath}
via fine analysis of the perturbations to the neural network outputs \(f^q_{\widetilde{\W}}(\x) - f^q_{\W}(\x)\) and \(f^k_{\widetilde{\theta}}(\x) - f^k_{\theta}(\x)\). In the rest of the proof, we apply techniques from NTK analysis to deal with the first order perturbations \(f^q_{\widetilde{\W}}(\x) - f^q_{\W}(\x) - \nabla_{\W} f^q_{\W}(\x)(\W')\), which eventually leads to:
\begin{align*}
    &\frac{1}{n}\sum_{i=1}^n \vbrack{\losstilde_i,\qt_i - q_i} - \vbrack{\nabla_{\W}L_S(\W,\theta),\W'}\leq  \norm{\losstilde}_2\cdot O\biggl(\frac{\omega^{1/3}L^2\sqrt{m\log m}}{\sqrt{nd}}\biggr)\norm{\W'}_F
\end{align*}
And similarly for the perturbations to \(f^k_{\theta}(\cdot)\). Combining these calculations completes the proof.

\subsection{Proof Sketch of Theorem \ref{thm:convergence}}

Equipped with \hyperref[lem:gradient-bounds-main]{Lemma \ref*{lem:gradient-bounds-main}} and \hyperref[lem:semi-smoothness-main]{Lemma \ref*{lem:semi-smoothness-main}}, we can sketch a proof of the convergence theorem of end-to-end training via gradient descent in unsupervised contrastive learning. 

\begin{proof}[Proof sketch of \hyperlink{thm}{Theorem \ref*{thm:convergence}}]
    Firstly we set the trajectory parameters as
    \begin{displaymath}
        \omega, \tau = O\bigl(n^{3.5}\sqrt{d}/(\delta\varepsilon\sqrt{m})\bigr)
    \end{displaymath} 
    and we assume that the parameters \(\W^{(t)},\theta^{(t)}\) in the training process always satisfy 
    \begin{displaymath}
        \W^{(t)} \in B(\W^{(0)},\omega),\, \theta^{(t)} \in B(\theta^{(0)},\tau)
    \end{displaymath}
    and we will justify this condition in trajectory analysis. Employing \hyperref[alg:GD]{Algorithm \ref*{alg:GD}}, we denote the gradient update at \(t\)-th iteration as
    \begin{displaymath}
        \nabla_{\W,t} = \nabla_{\W}L_S(\W^{(t)},\theta^{(t)}),\ \nabla_{\theta,t} = \nabla_{\theta}L_S(\W^{(t)},\theta^{(t)})
    \end{displaymath}
    Now from \hyperref[lem:semi-smoothness-main]{Lemma \ref*{lem:semi-smoothness-main}} and our choice of \(\omega, \tau, \eta, \gamma\), we can drop the second order terms in \eqref{ineq:semi-smoothness} and obtain
    \begin{align*}
        L_S(\W^{(t+1)},\theta^{(t+1)}) \leq & L_S(\W^{(t)},\theta^{(t)}) -\Omega(\eta)\norm{\nabla_{\W,t}}_F^2 -\Omega(\gamma)\norm{\nabla_{\theta,t}}_F^2\\
        & + O\biggl(\frac{\omega^{1/3}L^2\sqrt{m\log m}}{\sqrt{nd}}\biggr) \norm{\losstilde}_2\cdot\eta\norm{\nabla_{\W,t}}_F \\
        &+ O\biggl(\frac{\tau^{1/3}L^2\sqrt{m\log m}}{\sqrt{nd}}\biggr)\norm{\losshat}_2\cdot\gamma\norm{\nabla_{\theta,t}}_F
    \end{align*}
    From the gradient lower bound in \hyperref[lem:gradient-bounds-main]{Lemma \ref*{lem:gradient-bounds-main}} and our trajectory parameters \(\omega, \tau\), we can reduce the above inequality to 
    \begin{align*}
        L_S(\W^{(t+1)},\theta^{(t+1)})\leq & L_S(\W^{(t)},\theta^{(t)}) -\Omega\biggl(\frac{\eta \delta m}{n^3 d}\biggr)\norm{\losstilde^{(t)}}_2^2 -\Omega\biggl(\frac{\gamma \delta m}{n^3 d}\biggr)\norm{\losshat^{(t)}}_2^2\\
        \leq & - \Omega\biggl(\frac{ \delta m}{n^3 d}\min(\eta,\gamma)\biggr)\norm{\loss^{(t)}}_2^2
    \end{align*}
    where the last inequality is from our definition of \(\loss\)-vector in \hyperref[def:loss-vector-main]{Definition~\ref*{def:loss-vector-main}}. Now, by summing over \(t = 0,\dots,T-1\) and taking square root, and also by our choice of step sizes \(\eta, \gamma\), we can calculate
    \begin{align*}
        \frac{1}{T}\sum_{t=0}^{T-1}\norm{\loss^{(t)}}_2 & \leq O\biggl(\sqrt{\frac{n^3d}{T\min(\eta,\gamma)\delta m}}\biggr)\cdot \sqrt{L_S(\W^{(0)},\theta^{(0)}) - L_S(\W^{(T)},\theta^{(T)})}\\
        & \leq \frac{1}{\sqrt{T}}\cdot O\biggl(\frac{n^{5}L\sqrt{k}}{\delta^{3/2}\varepsilon}\biggr)
    \end{align*}
    So for \(T = \Theta(n^{10}L^2k/(\delta^3\varepsilon^4))\) iterations, we obtain
    \begin{displaymath}
        \frac{1}{T}\sum_{t=0}^{T-1}\norm{\loss^{(t)}}_2 \leq \varepsilon
    \end{displaymath}
    and in order for \(\W^{(t)}\) and \(\theta^{(t)}\) to stay in \(B(\W^{(0)},\omega)\) and \(B(\theta^{(0)},\tau)\) respectively, the over-parametrization needed would be \(m \geq \Omega(n^{12}L^{12}(\log m)^5 /(\delta^5 \varepsilon^2))\). The details of the above calculations and trajectory analysis is presented in \hyperref[appendix:proof-main-results]{Appendix \ref*{appendix:proof-main-results}}.
\end{proof}

\section{Conclusion and Future Work}

In this paper, we show that in unsupervised contrastive learning, end-to-end training via gradient descent can find an approximate stationary solution for the non-convex contrastive loss in polynomial time. Our proof is based on a careful analysis of the contrastive loss function and the gradient updates for two interactive deep neural networks, which allows us to analyze its optimization behavior. 

We discuss some directions for future research.

\begin{itemize}
    \item In \citet{arora2019theoretical} they established generalization bound for pretrained representations, but the representation is assumed to be frozen after pretraining (which means training only the top layer only). From our analysis of optimization, it would be possible to obtain a generalization bound that involves fine-tuning (which jointly trains all the layers).
    \item It would be of interest to know why minimizing the contrastive loss can lead to good feature representations. In the supervised setting, \citet{arora2019fine} and \citet{cao2019generalization} proved that the generalization performance of over-parameterized neural networks are closely related to their NTK. But since in contrastive learning we need two neural networks, their analysis cannot be trivially generalize to this setting.
\end{itemize}

\bibliography{main}
\bibliographystyle{plainnat}

\clearpage
\appendix

\begin{center}
    \LARGE\LARGE \textsc{Appendix}
\end{center}

\section{Proof of the Main Theorem}\label{appendix:proof-main-results}

First we restate the necessary definitions.

\begin{definition}[\(\loss\)-vectors]\label{def:loss-vector}
    For each \(i \in [n]\), we denote the gradients of our loss function to the outputs of both neural networks as
    \begin{align*}
        \losstilde = (\losstilde_i)_{i=1}^n = \bigl(\partial L_S/\partial q_i)\bigr)_{i \in [n]}  ,\qquad \losshat = (\losshat_i)_{i=1}^n = \bigl(\partial L_S/\partial k_i\bigr)_{i \in [n]}
    \end{align*}
    And we further define \(\loss = (\losstilde,\losshat)\) as our objective.
\end{definition}

We restate our main result of convergence here.
\hypertarget{thm-app}{}
\begin{theorem}[Convergence of Gradient Descent]\label{thm:convergence-appendix}
    For any \(\varepsilon \in (0,1)\), \(\delta \in (0,O(L^{-1}))\). Let 
    \begin{equation}
        m \geq \Omega\bigl(n^{12}L^{12}(\log m)^5\delta^{-5}\varepsilon^{-2}\bigr),\qquad \eta,\gamma = \Theta( d\varepsilon^2\delta^2/(n^{7}L^2km)), \qquad d\geq\Omega(\log^2 m)
    \end{equation}
    Suppose we do gradient descent at each iteration \(t = 0,1,\dots,T-1\). Then, with probability at least \(1 - O(m^{-1})\) over the random initialization, we have
    \begin{displaymath}
        \frac{1}{T}\sum_{t=0}^{T-1}\norm{\loss^{(t)}}_2 \leq \varepsilon \quad \text{for }T = \Theta\biggl(\frac{n^{10}L^{2}k}{\delta^3}\cdot \frac{1}{\varepsilon^4}\biggr) \text{ iterations}
    \end{displaymath}
    where the \(\loss^{(t)}\)-vectors are defined in Definition \ref{def:loss-vector}, with \(f^q\) and \(f^k\) parameterized by \(\W^{(t)}\) and \(\theta^{(t)}\) respectively.
\end{theorem}

\noindent We also restate the lemmas appeared in \hyperref[sec:proof-techniques]{Section \ref*{sec:proof-techniques}}.

\subsection{Main Technical Lemmas}

\begin{lemma}[Gradient Bounds]\label{lem:gradient-bounds}
    Let \(\omega, \tau \leq O(\delta^{3/2}/(n^{3}L^6(\log m)^{3/2}))\), with probability at least \(1 - e^{-\Omega(m\omega^{3/2}L)} - e^{-\Omega(m\tau^{3/2}L)}\)
    over the randomness of initialization, it satisfies for every \(\W \in B(\W^{(0)},\omega)\) and \(\theta\in B(\theta^{(0)},\tau)\), the following holds.
    \begin{itemize}
        \item For \(\norm{\nabla_\W L_S(\W,\theta)}_F\), we have 
        \begin{align*}
            \Omega\biggl(\frac{m\delta}{n^3d}\biggr)\sum_{i=1}^n\|\losstilde_i\|_2^2 \leq \Bigl\| \nabla_{\W}L_S(\W,\theta)\Bigr\|_F^2 \leq O\biggl(\frac{Lm}{nd}\biggr)\sum_{i=1}^n\|\losstilde_i\|_2^2
        \end{align*}
        \item For \(\norm{\nabla_{\theta}L_S(\W,\theta)}_F\), we have
        \begin{align*}
            \Omega\biggl(\frac{m\delta}{n^3 d}\biggr)\sum_{i=1}^n\|\losshat_i\|_2^2 \leq \Bigl\|\nabla_{\theta} L_S(\W,\theta)\Bigr\|_F^2 \leq O\biggl(\frac{Lm}{nd}\biggr) \sum_{i=1}^n\|\losshat_i\|_2^2
        \end{align*}
    \end{itemize}
\end{lemma}

\begin{lemma}[The Semi-smoothness Properties] \label{lem:semi-smoothness}
    For any \(\norm{\W'}_2\leq \omega\) and \(\norm{\theta'}_2\leq \tau\), where
    \begin{displaymath}
        \omega, \tau\in [\Omega(\sqrt{d/m}),\ O(1/ (L^{9/2} (\log m)^{3/2}) )]
    \end{displaymath} 
    Then we have, with probability at least \(1 - e^{-\Omega(m\omega^{3/2}L)}-e^{-\Omega(m\tau^{2/3}L)}\) over the randomness of initialization, the following inequality holds,
        \begin{align*}
            L_S(\W + \W',\theta + \theta')\leq\ & L_S(\W,\theta) + \vbrack{\nabla_{\W} L_S(\W,\theta),\W'} + \vbrack{\nabla_{\theta} L_S(\W,\theta),\theta'}\\
            & + O\biggl(\frac{\omega^{1/3}L^2\sqrt{m\log m}}{\sqrt{nd}}\biggr) \norm{\losstilde}_2\cdot\norm{\W'}_F \\
            &+ O\biggl(\frac{\tau^{1/3}L^2\sqrt{m\log m}}{\sqrt{nd}}\biggr)\norm{\losshat}_2\cdot\norm{\theta'}_F\\
            & + O\biggl(\frac{kL^2m^2}{d^2}\biggr)\Bigl(\tau^2\norm{\W'}_F^2+\omega^2\norm{\theta'}_F^2\Bigr)
        \end{align*}
\end{lemma}

\subsection{Proof of Theorem \ref{thm:convergence-appendix}}

\begin{proof}[Proof of Theorem \ref{thm:convergence-appendix}]
    We restate our parameter choice here for the convenience of readers:
    \begin{equation}\label{eq:parameter-choice}
        m \geq \Omega\bigl(n^{12}L^{12}(\log m)^5\delta^{-5}\varepsilon^{-2}\bigr),\qquad \eta,\gamma = \Theta( d\varepsilon^2\delta^2/(n^{7}kL^2m)), \qquad d\geq\Omega(\log^2 m)
    \end{equation}
    And we set the trajectory parameter 
    \begin{displaymath}
        \omega, \tau = O\biggl(\frac{n^{7/2}\sqrt{d}}{\delta\varepsilon\sqrt{m}}\biggr)
    \end{displaymath}
    which satisfy all the requirements in all the lemmas we have employed. In the proof below, we first 

    We denote \(\losstilde^{(t)}, \losshat^{(t)}\) as the \(\losstilde\) and \(\losshat\)-vectors where the query encoder \(f^q_{\W}\) and the key encoder \(f^k_{\theta}\) are parameterized by \(\W^{(t)}\) and \(\theta^{(t)}\) respectively. To perform gradient descent, we let the gradient update be 
    \begin{displaymath}
        \W^{(t+1)} = \W^{(t)} - \eta \nabla_{\W}L_S(\W^{(t)},\theta^{(t)}),\qquad \theta^{(t+1)} = \theta^{(t)} - \gamma \nabla_{\theta}L_S(\W^{(t)},\theta^{(t)})
    \end{displaymath}
    And for technical convenience we denote 
    \[\nabla_{\W,t} =  \nabla_{\W}L_S(\W^{(t)},\theta^{(t)}) \quad \text{and} \quad \nabla_{\theta, t} =  \nabla_{\theta}L_S(\W^{(t)},\theta^{(t)})\]
    Now from Lemma \ref{lem:semi-smoothness}, we can calculate 
    \begin{align*}
        L_S(\W^{(t+1)}, \theta^{(t+1)}) &\leq L_S(\W^{(t)},\theta^{(t)}) - \eta\|\nabla_{\W,t}\|_F^2 - \gamma\|\nabla_{\theta,t}\|_F^2\\
        & \quad + \norm{\losstilde^{(t)}}_2\cdot O\biggl(\frac{\omega^{1/3}L^2 \sqrt{m \log m}}{\sqrt{nd}}\biggr)\cdot\norm{\nabla_{\W,t}}_F\\
        & \quad + \norm{\losshat^{(t)}}_2\cdot O\biggl(\frac{\tau^{1/3}L^2 \sqrt{m \log m}}{\sqrt{nd}}\biggr)\cdot\norm{\nabla_{\theta,t}}_F \\
        & \quad + O\biggl(\frac{kL^2m^2}{d^2}\biggr)\cdot\Bigl(\eta^2\tau^2\norm{\nabla_{\W,t}}_2^2+ \gamma^2 \omega^2\norm{\nabla_{\theta,t}}_F^2\Bigr)
    \end{align*}
    Now from our step size choice \(\eta,\gamma = \Theta(d\varepsilon^2 \delta^2/(n^{7}L^2km))\) and our trajectory parameter choice \(\omega, \tau = O(n^{7/2}\sqrt{d}/(\delta\varepsilon\sqrt{m}))\), we can obtain 
    \begin{align*}
        L_S(\W^{(t+1)}, \theta^{(t+1)}) &\leq L_S(\W^{(t)},\theta^{(t)}) - \frac{3\eta}{4}\|\nabla_{\W,t}\|_F^2 - \frac{3\gamma}{4}\|\nabla_{\theta,t}\|_F^2\\
        & \quad + \norm{\losstilde^{(t)}}_2\cdot O\biggl(\frac{\omega^{1/3}L^2 \sqrt{m \log m}}{\sqrt{nd}}\biggr)\cdot\norm{\nabla_{\W,t}}_F\\
        & \quad + \norm{\losshat^{(t)}}_2\cdot O\biggl(\frac{\tau^{1/3}L^2 \sqrt{m \log m}}{\sqrt{nd}}\biggr)\cdot\norm{\nabla_{\theta,t}}_F \\
        & \leq  - \frac{3\eta}{4}\|\nabla_{\W,t}\|_F \biggl(\norm{\nabla_{\W,t}}_F -  O\biggl(\frac{\omega^{1/3}L^2 \sqrt{m \log m}}{\sqrt{nd}}\biggr)\cdot\norm{\losstilde^{(t)}}_2 \biggr) \\
        & \quad - \frac{3\gamma}{4}\|\nabla_{\theta,t}\|_F \biggl(\norm{\nabla_{\theta,t}}_F - O\biggl(\frac{\tau^{1/3}L^2 \sqrt{m \log m}}{\sqrt{nd}}\biggr)\cdot\norm{\losshat^{(t)}}_2\biggr)
    \end{align*}
    Now from Lemma \ref{lem:gradient-bounds} and our notation \(\nabla_{\W,t} = \nabla_{\W}L_S(\W^{(t)},\theta^{(t)}),\ \nabla_{\theta,t} = \nabla_{\theta}L_S(\W^{(t)},\theta^{(t)})\), we have
    \begin{displaymath}
        \norm{\nabla_{\W,t}}_F^2 \geq \Omega\biggl(\frac{m\delta}{n^3d}\biggr)\norm{\losstilde^{(t)}}_2^2,\qquad \norm{\nabla_{\theta,t}}_F^2 \geq \Omega\biggl(\frac{m\delta}{n^3d}\biggr)\norm{\losshat^{(t)}}_2^2
    \end{displaymath}
    We can choose \(\omega,\tau = O(\delta^{3/2}/(n^{3}L^6(\log m)^{3/2}))\) to ensure that 
    \begin{align*}
        L_S(\W^{(t+1)},\theta^{(t+1)}) - L_S(\W^{(t)},\theta^{(t)}) &\leq  -\Omega\biggl( \frac{\eta \delta m}{n^3d}\biggr) \norm{\losstilde^{(t)}}_2^2 - \Omega \biggl( \frac{\gamma \delta m}{n^3 d}\biggr)\norm{\losshat^{(t)}}_2^2 \\
        & \leq -\Omega\biggl(\frac{\min(\eta,\gamma)\delta m}{n^3d}\biggr)\norm{\loss^{(t)}}_2^2
    \end{align*}
    By averaging over \(t=0,\dots,T-1\), we arrive at 
    \begin{align} \label{eq:convergence}
        \frac{1}{T}\sum_{t=0}^{T-1}\norm{\loss^{(t)}}_2 \leq O\Biggl(\sqrt{\frac{n^3 d}{T\min(\eta,\gamma) \delta m}}\Biggr) \sqrt{L_S(\W^{(0)},\theta^{(0)})} \stackrel{\text{\ding{172}}}\leq O\Biggl(\sqrt{\frac{n^4 d}{T\min(\eta,\gamma) \delta m}}\Biggr)
    \end{align}
    where \ding{172} is due to the fact that, by Johnson-Lindenstrauss Lemma, with probability at least \(1-O(n)e^{-\Omega(d)}\), we have 
    \begin{displaymath}
        \|f^q_{\W^{(0)}}(\x_i)\|_2,\ \|f^k_{\theta^{(0)}}(\x_i)\|_2 \leq O(1)\quad \text{for all } i \in [n] \implies L_S(\W^{(0)},\theta^{(0)}) \leq O(\log k)\leq O(n)
    \end{displaymath}
    Thus for \(T\min(\eta,\gamma) =\Theta (n^3d/(m\delta\varepsilon^2))\), we have 
    \begin{displaymath}
        \frac{1}{T}\sum_{t=0}^{T-1}\norm{\loss^{(t)}}_2 \leq \varepsilon \quad\implies\quad \frac{1}{T}\sum_{t=0}^{T-1}\norm{\losstilde^{(t)}}_2 \leq \varepsilon,\quad \frac{1}{T}\sum_{t=0}^{T-1}\norm{\losshat^{(t)}}_2 \leq \varepsilon
    \end{displaymath}
    Note that from our choice of step sizes, \(\eta\) and \(\gamma\) are of the same order, which implies 
    \begin{displaymath}
        T\eta,\ T\gamma = \Theta (n^3d/(m\delta\varepsilon^2))
    \end{displaymath}
    Therefore the trajectory of \(\W\) satisfies
    \begin{align}\label{eq:trajectory}
        \begin{split}
        \norm{\W^{(t)} - \W^{(0)}}_F &\leq \sum_{t=0}^{T-1}\eta \norm{\nabla_{\W,t}}_F \leq  \sum_{t=0}^{T-1}\eta \sqrt{Lm/d} \cdot\norm{\losstilde^{(t)}}_2 \\
        &\leq \sqrt{T\eta}\cdot O\biggl(\sqrt{\frac{Lm}{d}\cdot \frac{n^4d}{m\delta}}\biggr) \leq O\biggl(\frac{n^{3.5}\sqrt{d}}{\delta\varepsilon\sqrt{m}}\biggr) = \omega    
        \end{split}
    \end{align}
    And similarly, the trajectory of \(\theta\) satisfies
    \begin{align*}
        \norm{\theta^{(t)} - \theta^{(0)}}_F &\leq \sum_{t=0}^{T-1}\gamma \norm{\nabla_{\theta,t}}_F \leq  \sum_{t=0}^{T-1}\gamma \sqrt{Lm/d} \cdot\norm{\losshat^{(t)}}_2 \\
        &\leq \sqrt{T\gamma}\cdot O\biggl(\sqrt{\frac{Lm}{d}\cdot \frac{n^4d}{m\delta}}\biggr) \leq O\biggl(\frac{n^{3.5}\sqrt{d}}{\delta\varepsilon\sqrt{m}}\biggr) = \tau
    \end{align*}
    And our final running time is 
    \begin{align*}
        T = \Theta \biggl(\frac{n^3d}{\min(\eta,\gamma) m \delta \varepsilon^2}\biggr) = \Theta\biggl(\frac{n^{10}L^{2}k}{\delta^3}\cdot \frac{1}{\varepsilon^4}\biggr)
    \end{align*}
\end{proof}

\section{Auxiliary Lemmas}

The lemmas in this section are adapted from \citet{allen2018convergence,zou2019improved} and modified to fit our setting. Note that all the lemmas are written with respect to the parameters \(\W\) of the query encoders \(f^q_{\W}\). They can be applied to the parameters \(\theta\) of the key encoders \(f^k_{\theta}\) as well. 

Firstly we define the following notations: Let \(\D_{i,l}\) be diagonal matrices defined as follows
\begin{align*}
        (\D_{i,l})_{r,r} = \mathds{1}\{ ([\W_l]_r \h_{i,l})\geq 0 \}\ r \in [m] \qquad \text{for } i\in [n],\ 0\leq l \leq L,
\end{align*}
where \([\W_l]_r\) is the \(r\)-th row of \(\W_l\). Now we can represent the outputs of hidden layers recursively as
\begin{align*}
    \h_{i,0} = \D_{i,0}\W_0 \x_i,\qquad \h_{i,l} = \D_{i,l}\W_l \h_{i,l-1} \text{ for } 1 \leq l \leq L-1,\qquad f_{\W}(\x_i) = \W_L \h_{i,L-1}
\end{align*}
For clarity we further define the product of matrices \(\A_i\) as
\begin{displaymath}
    \prod_{l=a}^b \A_l = \A_b\A_{b-1}\cdots \A_a\qquad \text{for matrices } \A_a,\dots,\A_b
\end{displaymath}
Specifically, we define the following notations for parameter \(\W\) at its random initialization \(\W^{(0)}\) (see \hyperref[def:initialization]{Definition \ref*{def:initialization}}). Set notations: for every \(i \in [n]\) and \(1 \leq l \leq L\), we define the matrices \(\D_{i,l}^{(0)}\) and vectors \(\h^{(0)}_{i,l}\) as
\begin{align*}
    \D_{i,l}^{(0)} := \diag\Bigl(\mathds{1}( \vbrack{[\W^{(0)}_{l}]_r, \h_{i,l-1}} \geq 0 ) \Bigr)_{k=1}^m, \quad\h_{i,0}^{(0)} := \D^{(0)}_{i,0}\W_0^{(0)}\x_i,  \quad \h^{(0)}_{i,l} = \D^{(0)}_{i,l}\W_l^{(0)}\h^{(0)}_{i,l-1}
\end{align*}

Now equipped with these notations, we can present the following technical lemmas

\begin{lemma}[Lemma 7.1 in \citet{allen2018convergence}] \label{Lemma-7.1}
    If \(\varepsilon \in (0,1)\), with probability at least \(1 - e^{-\Omega(\varepsilon^2m/L)}\) over the randomness of \(\W^{(0)}\), we have \(\norm{\h_{i,l}} \in [1-\varepsilon,1+\varepsilon]\) for all \(i\in [n]\) and \(l \in [L]\).
\end{lemma}

\begin{lemma}[Lemma 7.3 in \citet{allen2018convergence}]\label{Lemma-7.3}
    Suppose \(m \geq \Omega(nL\log(nL))\). With probability at least \(1 - e^{-\Omega(m/L)}\) over the randomness of initialization of \(\W^{(0)}\), for all \(i \in [n]\) and \(1 \leq a \leq b \leq L-1\)
    \begin{itemize}
        \item[(a)] \(\norm{\W_b^{(0)} \bigl(\prod_{l=a}^{b-1}\D^{(0)} _{i,l}\W^{(0)} _l\bigr)}_2 \leq O(\sqrt{L})\).
        \item[(b)] \(\norm{\W_b^{(0)} \bigl(\prod_{l=a}^{b-1}\D^{(0)} _{i,l}\W^{(0)} _l\bigr)v}_2 \leq 2\norm{v}_2\) for all \(v \in \R^{m}\) with \(\norm{v}_0 \leq O(\frac{m}{L\log m})\).
        \item[(c)] \(\norm{u^{\top}\W_b^{(0)} \bigl(\prod_{l=a}^{b-1}\D^{(0)} _{i,l}\W^{(0)} _l\bigr)}_2 \leq O(1)\norm{u}_2\) for all \(u \in \R^{m}\) with \(\norm{u}_0 \leq O(\frac{m}{L\log m})\).
        \item[(d)] For any integer \(1 \leq s \leq O(\frac{m}{L\log m})\), with probability at least \(1 - e^{-\Omega(s\log m)}\) over the randomness of initialization, we have \(|u^{\top}\W_b^{(0)} \bigl(\prod_{l=a}^{b-1}\D^{(0)} _{i,l}\W^{(0)} _l\bigr)v| \leq O(\sqrt{\frac{s\log m}{m}})\norm{u}_2\norm{v}_2\) for all vectors \(u,v \in \R^{m}\) with \(\norm{u}_0,\norm{v_0} \leq s\).
    \end{itemize}
\end{lemma}

\begin{lemma}[backward propagation]\label{lem:backward-propagate}
    Suppose \(m \geq \Omega(nL\log (nL))\), \(\Omega(\frac{d}{\log m}) \leq s \leq O(\frac{m}{L \log m})\) and \(d \leq O(\frac{m}{L\log m})\), then for all indices \( i \in [n]\), \(1 \leq a \leq L-1\),
    \begin{itemize}
        \item[(a)] with probability at least \(1 - e^{\Omega(s \log m)}\), for all \(v \in \R^d\) such that \(\norm{v}_0 \leq s\), we have 
        \begin{displaymath}
            \Bigl|u^{\top}\W^{(0)} _L\Bigl(\prod_{l=a}^{L}\D^{(0)} _{i,l}\W^{(0)} _l\Bigr)v\Bigr| \leq O\Bigl(\sqrt{\frac{s\log m}{m}}\Bigr)\norm{u}_2\norm{v}_2
        \end{displaymath}  
        \item[(b)] with probability at least \(1 - e^{-\Omega(m/L)}\), for all vectors \(u \in \R^d\), we have 
        \begin{displaymath}
            \Bigl\|u^{\top}\W^{(0)} _L\Bigl(\prod_{l=a}^{L}\D^{(0)} _{i,l}\W^{(0)} _l\Bigr)\Bigr\|_2 \leq O(\sqrt{m/d})\cdot\norm{u}_2
        \end{displaymath} 
    \end{itemize}
\end{lemma}

\begin{lemma}[Lemma 8.2(b), 8.2(c) in \citet{allen2018convergence}]\label{Lemma-8.2}
    Suppose \(\omega \leq O(L^{-9/2}(\log m)^3)\), with probability at least \(1-e^{-\Omega(m\omega^{2/3}L)}\), for every \(\W\) such that \(\norm{\W - \W^{(0)}}_2\leq \omega\):
    \begin{itemize}
        \item[(b)] Let the diagonal matrices \(\D_{i,l}\), \({\D_{i,l}^{(0)}}\) and \(\D'_{i,l}\) be defined as 
        \begin{displaymath}
            (\D_{i,l})_{k,k} = \mathds{1}\{(\W_l\h_{i,l-1})_k \geq 0\},\quad (\D_{i,l}^{(0)})_{k,k} = \mathds{1}\{(\W^{(0)}_l\h^{(0)}_{i,l-1})_k \geq 0\},\quad \D'_{i,l} = \D_{i,l} - \D^{(0)}_{i,l}
        \end{displaymath} 
        we have \(\norm{\D'_{i,l}}_0 \leq O(m\omega^{2/3}L)\) and \(\norm{\D'_{i,l}\W_{l}\h_{i,l-1}}_2 \leq O(\omega L^{3/2})\).
        \item[(c)] \(\norm{\h_{i,l}-\h^{(0)}_{i,l}}_2\leq O(\omega L^{5/2}\sqrt{\log m})\).
    \end{itemize}
\end{lemma}

We present a modified lemma on the perturbation analysis of intermediate layers with respect to small changes of parameters. Note that in our paper, the last hidden layer is the \(L-1\)-th layer. 

\begin{lemma}[Modification of Lemma 8.6 in \citet{allen2018convergence}]\label{Lemma-8.6}
    For any interger \(s\) such that \(1 \leq s \leq O(\frac{m}{L^3\log m})\), with probability at least \(1 - e^{-\Omega(s\log m)}\) over the randomness of initialization,
    \begin{itemize}
        \item for every \( i \in [n]\) and \(0 \leq a \leq b \leq L\)
        \item for every diagonal matrices \(\D''_{i,0},\dots,\D''_{i,L-1} \in [-3,3]^{m\times m}\) with at most \(s\) non-zero entries.
        \item for every perturbation matrices \(\W' = (\W'_0,\W'_1,\dots,\W'_{L-1})\in (\R^{m\times \mathfrak{b}},\R^{(m\times m)L})\) with \(\norm{\W}_2 \leq \omega \in [0,1]\).
    \end{itemize}
    We have 
    \begin{itemize}
        \item[(a)] \(\|\mathbf{W}_{b}^{(0)}(\mathbf{D}_{i, b-1}^{(0)}+\mathbf{D}_{i, b-1}^{\prime \prime}) \cdots(\mathbf{D}_{i, a}^{(0)}+\mathbf{D}_{i, a}^{\prime \prime}) \mathbf{W}_{a}^{(0)} \|_{2} \leq O(\sqrt{L})\).
        \item[(b)] \(\|(\mathbf{W}_{b}^{(0)}+\mathbf{W}_{b}^{\prime})(\mathbf{D}_{i, b-1}^{(0)}+\mathbf{D}_{i, b-1}^{\prime \prime}) \cdots(\mathbf{D}_{i, a}^{(0)}+\mathbf{D}_{i, a}^{\prime \prime})(\mathbf{W}_{a}^{(0)}+\mathbf{W}_{a}^{\prime})\|_{2} \leq O(\sqrt{L})\) if \(\omega \leq O(L^{-3/2})\).
    \end{itemize}
\end{lemma}

\begin{proof}
    The only difference of this lemma and Lemma 8.6 in \citet{allen2018convergence} is that we have taken into account the first layer \(\W_0 \in \R^{m\times \mathfrak{b}}\). Actually we can go through the same procedure as in Lemma 7.3 in \citet{allen2018convergence} to give a bound \(\|\mathbf{W}_{b}^{(0)}\mathbf{D}_{i, b-1}^{(0)} \cdots\mathbf{D}_{i, a}^{(0)}\mathbf{W}_{0}^{(0)} \|_{2} \leq O(\sqrt{L})\) with probability at least \(1-e^{-\Omega(m/L)}\). Then with the same techniques in the proof of Lemma 8.6 in \citet{allen2018convergence}, we obtain the same result.
\end{proof}

Equipped with this lemma, we are now ready to give our version of backward perturbation lemma, which takes into account both the first layer and the last layer.

\begin{lemma}[Modification of Lemma 8.7 in \citet{allen2018convergence}]\label{lem:backward-perturbation}
    Suppose \(d \leq O(\frac{m}{L \log m}))\),
    \begin{itemize}
        \item for any integer \(s\) such that \(\Omega(\frac{d}{\log m}) \leq s \leq O(\frac{m}{L^3\log m})\),
        \item for all \(i \in [n]\) and \(1 \leq a \leq L\),
        \item for every diagonal matrices \(\D''_{i,0},\dots,\D''_{i,L-1} \in [-3,3]^{m\times m}\) with at most \(s\) non-zero entries,
        \item for every perturbation matrices \(\W' = (\W'_{i,0},\dots,\W'_{i,L})\) with \(\norm{\W'}_2 \leq \omega = O(\frac{1}{L^{3/2}})\),
    \end{itemize}
    it satisfies, with probability at least \(1 - e^{-\Omega(s \log m)}\) over the randomness of initialization,
    \begin{align*}
        \biggl\| (\W^{(0)}_{L} + \W'_{L})\biggl(\prod_{l=a+1}^L(\D^{(0)}_{i,l}+\D''_{i,l})(\W_{l}^{(0)} + \W'_{l})\biggr)(\D^{(0)}_{i,a} + \D''_{i,a}) - \W^{(0)}_{L}\biggl(\prod_{l=a+1}^L\D^{(0)}_{i,l}\W_{l}^{(0)} \biggr)\D^{(0)}_{i,a}\biggr\|_2\\
        \leq O(\sqrt{L^3s\log m/d}+\omega\sqrt{L^3m/d})
    \end{align*}
    Note that if \(s = O(m\omega^{2/3}L)\), this perturbation bound becomes \(O(\omega^{1/3}L^2\sqrt{m\log m/d})\).
\end{lemma}

\begin{proof}
    For notational simplicity we ignore subscripts for \(i\) in the proof. Now we compute
    \begin{align*}
        &\biggl\| (\W^{(0)}_{L} + \W'_{L})\biggl(\prod_{l=a+1}^L(\D^{(0)}_{l}+\D''_{l})(\W_{l}^{(0)} + \W'_{l})\biggr)(\D^{(0)}_{a} + \D''_{a}) - \W^{(0)}_{L}\biggl(\prod_{l=a+1}^L\D^{(0)}_{l}\W_{l}^{(0)} \biggr)\D^{(0)}_{a}\biggr\|_2\\
        & \leq \sum_{l=a}^{L-1} \underbrace{\biggl\|\W^{(0)}_L\biggl( \prod_{b=l+1}^{L} \D^{(0)}_{b}\W^{(0)}_b \biggr)\biggr\|_2}_{\text{\ding{172}}} \norm{\D''_{l}}_2 \underbrace{\biggl\|  \biggl(\prod_{c=a}^{l}(\W_{c}^{(0)} + \W'_{c}) (\D^{(0)}_{c} + \D''_{c}) \biggr) \biggr\|_2}_{\text{\ding{173}}}\\
        & \quad + \sum_{l=a}^{L} \underbrace{\biggl\|\W^{(0)}_L  \prod_{b=l+1}^{L-1} (\D^{(0)}_{b}\W^{(0)}_b)\D^{(0)}_{l}\biggr\|_2}_{\text{\ding{174}}}\norm{\W'_{l}}_2 \norm{\D^{(0)}_{l} + \D''_{l}}_2 \underbrace{\biggl\|  \biggl(\prod_{c=a}^{l-1}(\W_{c}^{(0)} + \W'_{c}) (\D^{(0)}_{c} + \D''_{c}) \biggr)\biggr\|_2}_{\text{\ding{175}}}\\
        & \leq L\cdot O\biggl(\sqrt{\frac{s\log m}{d}}\cdot\sqrt{L}\biggr) + L\cdot O\biggl(\sqrt{m/d}\cdot\omega\cdot\sqrt{L}\biggr) = O(\sqrt{L^3s\log m/d}+\omega\sqrt{L^3m/d})
    \end{align*}
    where \ding{172} is from Lemma \ref{lem:backward-propagate}(a) and the fact that \(\D''_{l}\prod_{c=a}^{l}(\W_{c}^{(0)} + \W'_{c}) (\D^{(0)}_{c} + \D''_{c})\) is a \(s\)-sparse matrix; \ding{173} is from Lemma \ref{Lemma-8.6}(b); \ding{174} is from Lemma \ref{lem:backward-propagate}(b); \ding{175} is again from Lemma \ref{Lemma-8.6}(b).
\end{proof}

To conclude this section, we modify the Claim 11.2 in \citet{allen2018convergence} to fit our setting.

\begin{lemma}\label{lem:output-perturbation} 
    Let \(\W \in B(\W^{(0)},\omega)\) and \(\W' = (\W'_0,\W'_1,\dots,\W'_L)\) be such that \(\norm{\W'}_2\leq \omega\), where \(\omega \leq O(\frac{1}{L^6n^{3}(\log m)^{3/2}})\). Denote
    \begin{displaymath}
        \h_{i,0} = \sigma(\W_0\x_i),\quad \h_{i,l} = \sigma(\W_l\h_{i,l-1}),\quad \h'_{i,0} = \sigma((\W_0+\W'_0)\x_i),\quad \h'_{i,l} = \sigma((\W_l+\W'_l)\h'_{i,l-1})
    \end{displaymath}
    Then their exist diagonal matrices \(\D''_{i,l} \in \R^{m\times m}\) with entries in \([-1,1]\) such that, for any \(i \in [n]\) and \(0 \leq l \leq L-1\),
    \begin{displaymath}
        \h'_{i,l} - \h_{i,l} = \sum_{a=1}^l\biggl( \prod_{b=a+1}^l (\D_{i,l}+\D''_{i,l})\W_l\biggr)  (\D_{i,a}+\D''_{i,a})\W'_a \h'_{i,a-1}
    \end{displaymath}
    Further more, with probability at least \(1-e^{-\Omega(m\omega^{2/3}L)}\), we have 
    \begin{itemize}
        \item \(\norm{\h'_{i,l} - \h_{i,l}}_2 \leq O(L^{3/2})\norm{\W'}_2\),
        \item \(\norm{f_{\W+\W'}(\x_i) - f_{\W}(\x_i)}_2 \leq O(L\sqrt{m/d})\norm{\W'}_2\),
        \item \(\norm{\D''_{i,l}}_0 \leq O(m\omega^{2/3}L)\).
    \end{itemize}
\end{lemma}

Before we came to the proof of Lemma \ref{lem:output-perturbation}, we present the following auxiliary lemma.

\begin{lemma}[Proposition 11.3 in \citet{allen2018convergence}] \label{Proposition-11.3}
   Given vectors \(a,b \in \R^m\) and diagonal matrices \(\D\) where \(\D_{k,k} = \mathds{1}_{a_k\geq 0}\). Then, there exist a diagonal matrix \(\D'' \in \R^m\) with 
   \begin{itemize}
       \item \(|\D_{k,k}-\D''_{k,k}|\leq 1\) and \(|\D''_{k,k}|\leq 1\) for \(k \in [m]\),
       \item \(\D''_{k,k} \neq 0\) only when \(\mathds{1}_{a_k \geq 0} \neq \mathds{1}_{b_k \geq 0}\),
       \item \(\sigma(a) - \sigma(b) = (\D+\D'')(a-b)\).
   \end{itemize}
\end{lemma}

\begin{proof}[Proof of Lemma \ref{lem:output-perturbation}]
    The proof is almost the same with the proof of Claim 11.2 in \citet{allen2018convergence}, and we do not repeat most of its content here. The only difference in our claim is that we consider the training of the first and the last layer. We prove the part of  here. Ignore subscripts of \(i\) for simplicity, we calculate
    \begin{align*}
        &\|(\W_L+\W'_L)\h'_{L-1} - \W_L\h_{L-1}\|_2 \\
        = \ & \|\W'_L\h'_{L-1} + \W_L(\h'_{L-1} - \h_{L-1}) \|_2\\
        = \ & \|\W'_L\h'_{L-1} + \W'_L(\sigma((\W_{L-1}+\W_{L-1}')\h'_{L-2}) - \sigma(\W_l\h_{L-2})) \|_2\\
        \stackrel{\text{\ding{172}}}= \ & \|\W'_L\h'_{L-1} + \W_L(\D_{L-1}+\D''_{L-1})((\W_{L-1}+\W_{L-1}')\h'_{L-2}-\W_{L-1}\h_{L-2})\|_2 \\
        \leq \ & \|\W'_L\h'_{L-1}\|_2 + \|\W_L(\D_{L-1}+\D''_{L-1})\W'_{L-1}\h'_{L-2}\|_2\\
        & + \|\W_L(\D_{L-1}+\D''_{L-1})\W_{L-1}(\h'_{L-2} - \h_{L-2}) \|_2\\
        = \ & \underbrace{\|\W'_L\h'_{L-1}\|_2}_{\leq \norm{\W'_L}_2\norm{\h'_{l-1}}_2 } + \underbrace{\biggl\|\sum_{l=0}^{L-1} \W_L\biggl( \prod_{a=l+1}^{L-1} (\D_{a}+\D''_{a})\W_a\biggr)(\D_{l}+\D''_{l})\biggr\|_2}_{\leq O(L\sqrt{m/d}) \text{ by Lemma \ref{lem:backward-propagate} and Lemma \ref{lem:backward-perturbation}}} \cdot \norm{\W'_l\h'_{l-1}}_2\\
        \leq \ & O(L\sqrt{m/d})\norm{\W'}_2
    \end{align*}
    where in \ding{172} we have used Lemma \ref{Proposition-11.3}. And in the last inequality we have used Lemma \ref{Lemma-8.2}(c) to give \(\norm{\h'_{l}}_2 \leq \norm{\h^{(0)}_{l}}_2 + \norm{\h'_{l}-\h^{(0)}_l}_2 \leq O(1)\) for all \(0\leq l \leq L-1\).
\end{proof}

Combine Lemma \ref{lem:backward-propagate} Lemma \ref{lem:output-perturbation} together, we have a corollary.

\begin{corollary}[output-boundedness]\label{lem:output-boundedness}
    Let \(\W \in B(\W^{(0)},\omega)\), where \(\omega\) meets all the requirement in previous Lemmas and \(d \leq O(m/(L\log m))\), with probability at least \(1 - O(n)e^{-\Omega(d)}\), we have \(\norm{f_{\W^{(0)}}(\x_i)}_2 \leq O(1)\) and \(\norm{f_{\W}(\x_i)}_2 \leq O(1+\omega L\sqrt{m/d})\) for all \(i \in [n]\).
\end{corollary}

\begin{proof}
    Firstly, from Lemma \ref{Lemma-7.1} we know that, with probability at least \(1 - O(nL)e^{-\Omega(m/L)}\) we have 
    \begin{displaymath}
        \|\sigma(\W^{(0)}_{L-1}\sigma(\cdots\sigma(\W_0\x_i)))\|_2 \leq O(1)
    \end{displaymath}
    Conditioning on this event, since \((\W^{(0)}_L)_{i,j} \sim \N(0,\frac{1}{d}),\ (i,j)\in [d]\times [m]\), we have, over the randomness of \(\W_L^{(0)}\),
    \begin{displaymath}
        f^q_{\W^{(0)}}(\x_i) = \W_L\sigma(\cdots\sigma(\W_0\x_i)) \sim \N(0,\kappa^2\id_d)
    \end{displaymath}
    where \(\kappa^2 \leq O(\frac{1}{d})\). Therefore, with probability at least \(1 - O(n)e^{-\Omega(d)}\) over the initialization, we have
    \begin{displaymath}
        \norm{f^q_{\W^{(0)}}(\x_i)}_2 \leq O(1)
    \end{displaymath}
    and then apply Lemma \ref{lem:output-perturbation} to bound the perturbation of \(\W' = \W-\W^{(0)}\), where we have assumed \(\norm{\W'}_2\leq \omega\).
\end{proof}

Finally we present the \(\delta\)-separateness lemma in \citet{allen2018convergence}.

\begin{lemma}\label{lem:separateness}[]
    Suppose \(\delta \leq O(1/L)\), for every \(i\neq j\) and every layer \(l \in [L]\), we have with probability at least \(1-O(n^2)e^{-\Omega(m\delta^4)}\) over the initialization,
    \begin{displaymath}
        \norm{\h_{i,l}^{(0)} - \h_{j,l}^{(0)}}_2 \geq \delta/2
    \end{displaymath}
\end{lemma}

\begin{proof}
    We prove the lemma via induction. Suppose at layer \(l-1\) we have \(\delta_{l-1}\)-separateness, that is 
    \begin{displaymath}
        \norm{\h_{i,l-1}^{(0)} - \h_{j,l-1}^{(0)}}_2 \geq \delta_{l-1}
    \end{displaymath}
    for some \(\delta_{l-1}\geq \delta/2\). We try to prove that it still holdes for layer \(l\). Denote $\w_{l,r}$ to be the \(r\)-th row of \(\W_l\) at \(l\)-th layer, where \(\w_{l,r} \in \R^{1\times m}\) are row vectors, following the distribution \(N(0,\frac{2}{m}\id)\). Then over the randomness of \(\W_l\) and fix \(\h_{i,l-1},\h_{j,l-1}\), we have that \(\w_{l,r}\h_{i,l-1},\w_{l,r}\h_{j,l-1}\) are two mean zero Gaussian variables (though they are not independent). Therefore \(\sigma(\w_{l,k}\h_{i,l-1}) - \sigma(\w_{l,r}\h_{j,l-1})\) may have four different output. Now we ignore the subscript of layer \(l-1\) for simplicity and write
    \begin{displaymath}
        \sigma(\w_{r}\h_{i}) - \sigma(\w_{r}\h_{j}) = \begin{cases}
            \w_{r}(\h_{i} -\h_{j}), &\qquad\diamondsuit\ \text{ if both } \w_{r}\h_{i},\ \w_{r}\h_{j} \geq 0\\
            0, &\qquad\clubsuit\ \text{ if both } \w_{r}\h_{i},\ \w_{r}\h_{j} \leq 0\\
            \w_{r}\h_{i}, &\qquad\heartsuit\ \text{ if }  \w_{r}\h_{i} \geq 0,\ \w_{r}\h_{j} \leq 0\\
            \w_{r}\h_{j}, &\qquad\spadesuit\ \text{ if }  \w_{r}\h_{i} \leq 0,\ \w_{r}\h_{j} \geq 0
        \end{cases}
    \end{displaymath}
    In the case \(\diamondsuit\), we have 
    \begin{displaymath}
        \E[(\sigma(\w_{r}\h_{i}) - \sigma(\w_{r}\h_{j}))^2|\diamondsuit] \geq \frac{2\delta^2}{m}
    \end{displaymath}
    from our inductive assumption. In the case \(\clubsuit\), we have 
    \begin{displaymath}
        \E[(\sigma(\w_{r}\h_{i}) - \sigma(\w_{r}\h_{j}))^2|\clubsuit] = 0
    \end{displaymath}
    In the case \(\heartsuit\) and \(\spadesuit\), we have from Lemma \ref{Lemma-7.1}, \(\norm{\h_{i}}_2,\norm{\h_j}_2 \in [1/2,2]\) with high probability. Therefore we can calculate
    \begin{displaymath}
        \E[(\sigma(\w_{r}\h_{i}) - \sigma(\w_{r}\h_{j}))^2|\heartsuit\lor \spadesuit] \geq \frac{2\delta^2}{m} 
    \end{displaymath}
    Notice that the probability of the event \(\clubsuit\) is no more than \(1/2\) (for fixed \((i,j)\)-pair). So we obtain 
    \begin{displaymath}
        \E[(\sigma(\w_{r}\h_{i}) - \sigma(\w_{r}\h_{j}))^2] \geq \frac{\delta^2}{m},\qquad \E[ (\sigma(\w_{r}\h_{i}) - \sigma(\w_{r}\h_{j}))^2] \leq 2
    \end{displaymath}
    Now pick up the subscripts for layer \(l\), via Chernoff bound, we have, with probability at least \(1 - e^{-\Omega(m\delta^4)}\),
    \begin{displaymath}
        \|\h_{i,l} - \h_{j,l}\|_2^2 \geq \sum_{r=1}^m (1-O(\delta))\E[(\sigma(\w_{r}\h_{i}) - \sigma(\w_{r}\h_{j}))^2] \geq \delta^2(1- O(\delta))
    \end{displaymath}
    then we can take a union bound over all \((i,j)\)-pair, and proceed induction step over all layer \(0\leq l\leq L\) to conclude the proof.
\end{proof}
\section{Proof of Gradient Bounds}\label{sec:gradient-bounds}

\subsection{Key Calculations}\label{subsec:key-calculations}
 
\begin{itemize}
    \item For a matrix \(\W\) or \(\theta\), we denote \([\W]_r\) or \([\theta]_r\) their \(r\)-th row, \([\W]^r\) or \([\theta]^r\) their \(r\)-th column.
    \item For the query encoder \(f^q_{\W}\), we define \(\back_{i,l}^q:= \W_L\D_{i,L-1}\cdots\D_{i,l}\W_l\), \(\back^q_{i,L} = \W_L\)
    \item For the key encoder \(f^k_{\theta}\), we define \(\back_{i,l}^k:= \theta_L\D_{i,L-1}\cdots\D_{i,l}\theta_l\), \(\back^k_{i,L}=\theta_L\).
    \item For gradient \(\nabla_{\W}L_S(\W,\theta)\) with respect to \(\W\), we have
    \begin{displaymath}
        \nabla_{[\W_l]_r}L_S(\W,\theta) := \frac{1}{n}\sum_{i=1}^n \Bigl((\back_{i,l+1}^q)_r^{\top}\losstilde_i\Bigr)\cdot\sigma'(\vbrack{[\W_l]_{r},\h_{i,l-1}})\cdot\h_{i,l-1} 
    \end{displaymath}
    where \(\losstilde_i\) is defined as
    \begin{displaymath}
        \losstilde_i := \E^{Neg(i)}\Biggl[ \sum_{j=1}^k\frac{\exp(q_i^{\top}z_{i,j})}{1 + \sum_{s=1}^k\exp(q_i^{\top}z_{i,j})}z_{i,j} \Biggr]
    \end{displaymath}
    and \(q_i:= f^q_{\W}(\x_i),\ z_{i,j} = f^k_{\theta}(\x_{i,j}) - f^k_{\theta}(\x_i)\).
    \item For gradient \(\nabla_{\theta}L_S(\W,\theta)\) with respect to \(\theta\), we carefully compute
    \begin{align*}
        \nabla_{[\theta_l]_r} L_S(\W,\theta) & =  \frac{1}{n}\sum_{i=1}^n \E^{Neg} \biggl[\nabla_{[\theta_l]_r}\ell(\W,\theta,\x_i,\{\x_{i,j}\}_{j=1}^k)\biggr]\\
        =  \frac{1}{n}\sum_{i=1}^n &\frac{1}{\binom{n-1}{k}} \sum_{\{\x_{i,j}\}_{j=1}^k \subset S^{\setminus i}} \sum_{j=1}^k \frac{\exp(q_i^{\top}z_{i,j})}{1 + \sum_{s=1}^k \exp(q_i^{\top} z_{i,s})} \nabla_{[\theta_l]_r}\Bigl(q_i^{\top} (f^k_{\theta}(\x_{i,j}) - f^k_{\theta}(\x_{i}))\Bigr)
    \end{align*}
    To handle this complex summation, we introduce the notation \(\losshat(\x_i,\x_j)\) as the \(\loss\) vector (corresponding to \(\theta\)) which only contains \(f^q_{\W}(\x_i)\) and \(f^k_{\theta}(\x_j)\) in the nominator of the coefficients:
    \begin{equation}\label{def:losshat}
        \losshat(\x_i,\x_j) := \frac{1}{\binom{n-1}{k}}\sum_{\x_j \in \{\x_{i,s}\}_{s\in [k]} \subset S^{\setminus i} } \frac{\exp(q_i^{\top}z_{j})}{1 + \sum_{s=1}^k \exp(q_i^{\top} z_{i,s})}\cdot q_i
    \end{equation}
    where \(q_i = f^q_{\W}(\x_i),\ z_j = f^k_{\theta}(\x_{j}) - f^k_{\theta}(\x_i),\ z_{i,s} = f^k_{\theta}(\x_{i,s}) - f^k_{\theta}(\x_i) \). Then we can rearrange terms in \(\nabla_{[\theta_l]_r} L_S(\W,\theta)\) to get
    \begin{align}
        &\nabla_{[\theta_l]_r} L_S(\W,\theta) \nonumber \\
        = \ & \frac{1}{n}\sum_{i=1}^n \sum_{j\neq i}  \biggl( (\back_{i,l}^k)^{\top} \losshat(\x_i,\x_j)\biggr) \Bigl(\sigma'(\vbrack{[\theta_l]_r,\h_{j,l-1}})\h_{j,l-1}-\sigma'(\vbrack{[\theta_l]_r,\h_{i,l-1}})\h_{i,l-1}\Bigr) \nonumber\\
        = \ & \frac{1}{n}\sum_{i=1}^n  \Bigl((\back_{i,l}^k)^{\top} \Bigl(\sum_{j\neq i}(\losshat(\x_j,\x_i) - \losshat(\x_i,\x_j))\Bigr)\Bigr)\cdot \sigma'(\vbrack{[\theta_l]_r,\h_{i,l-1}})\cdot\h_{i,l-1} \nonumber\\
        = \ &  \frac{1}{n}\sum_{i=1}^n  \Bigl((\back_{i,l}^k)^{\top} \losshat_i \Bigr)\cdot \sigma'(\vbrack{[\theta_l]_r,\h_{i,l-1}})\cdot\h_{i,l-1} \label{eq:gradient-for-theta}
    \end{align}
    where \(\losshat_i := \sum_{j\neq i}(\losshat(\x_j,\x_i) - \losshat(\x_i,\x_j))\). This form \eqref{eq:gradient-for-theta} of \(\nabla_{[\theta_l]_r} L_S(\W,\theta)\) will facilitate our calculations in the proofs in Subsection \ref{subsec:gradient-bound-init}.
\end{itemize}

\subsection{Lemma of Gradient Lower Bound}

We present our lemma of gradient lower bound at initialization here, where the only difference of our lemma and the Lemma B.2 in \citet{zou2019improved} is that we have a probability bound \(1 - e^{-\Omega(m\delta^2/n^2)}\) instead of \(1 - e^{-\Omega(m\delta/nd)}\).

\begin{lemma}\label{lem:gradient-lower-bound}
    Assume \(m \geq \Omega(n^3d\delta^{-2})\), Let \(\W_L\) and \(\W_{L-1}\) be at random initialization, then with probability at least \(1 - e^{-\Omega(m\delta^2/n^2)}\) for any vectors \(\v_i, i \in [n]\), it holds that 
    \begin{displaymath}
        \sum_{r=1}^m \biggl\|\frac{1}{n}\sum_{i=1}^n \vbrack{[\W_L]^r, \v_i}\sigma'(\vbrack{[\W_{L-1}]_r, \h_{i,L-2}})\h_{i,L-2}\biggr\|_2 \geq \Omega\biggl(\frac{m\delta}{n^3d}\biggr)\sum_{i=1}^n\norm{\v_i}_2^2
    \end{displaymath}
\end{lemma}

Before we state the technical lemmas for the proof of Lemma \ref{lem:gradient-lower-bound}, we introduce the notations in \citet{zou2019improved}. Let \(\h_1,\dots,\h_n \in \R^m\) such that \(1/2\leq \norm{\h_i}_2 \leq 2\). Let \(\bar{\h_i}:= \h_i/\norm{\h_i}_2\) and assume \(\norm{\bar{\h}_i - \bar{\h}_j}_2 \geq \delta/2\) (from Lemma \ref{lem:separateness} we know this holds with high probability). Now we construct orthonormal matrices \(\Q_i = [\bar{\h}_i,\Q'_i] \in \R^{m\times m}\). For a standard gaussian random vector \(\w \sim \N(0,\id_m)\), we decompose \(\w = \Q_i\u_i = u_{i,1}\bar{\h}_i + \Q'_i\u'_i\), where \(u_{i,1}\) is the first entry of \(\u_i\) and \(\u'_i = (u_{i,2},\dots,u_{i,m}) \in \R^{m-1}\). Let \(\xi = \sqrt{\pi}\delta / (16n)\), define the following event over the randomness of \(\w\):
\begin{displaymath}
    W_i = \{ |u_{i,1}| \leq \xi,  |\vbrack{\Q'_i\u'_i,\bar{\h}_j}| \geq 2\xi \text{ for all } \bar{\h}_j \text{ where } j \neq i\}
\end{displaymath}
Then we have 
\begin{lemma}[Lemma C.1 in \citet{zou2019improved}]\label{Lemma-C.1}
    For each \(W_i\) and \(W_j\), we have 
    \begin{displaymath}
        \P(\w \in W_i) \geq \frac{\delta}{n16\sqrt{2e}} \quad \text{and} \quad W_i\cap W_j = \varnothing
    \end{displaymath}
\end{lemma}

Now we present two lemmas for technical purposes.

\begin{lemma}[Lemma C.2 in \citet{zou2019improved}]\label{Lemma-C.2}
    For any numbers \(a_1,\dots,a_n\), let
    \begin{displaymath}
        \h(\w):= \sum_{i=1}^n a_i \sigma'(\vbrack{\w,\h_i})\h_i
    \end{displaymath}
    where \(\w \sim \N(0,\id_m)\). It holds that 
    \begin{displaymath}
        \P\Bigl( \norm{\h(\w)}_2 \geq \frac{|a_i|}{4} \Big| \w \in W_i \Bigr) \geq 1/2
    \end{displaymath}
\end{lemma}

\begin{proof}[Proof of Lemma \ref{lem:gradient-lower-bound}]
    Fix \(\v_1,\dots,\v_n\). For \(r \in [m]\), define the function \(\h_r\) as
    \begin{displaymath}
        \h_r(\W_L,\W_{L-1}) := \sum_{i=1}^n \vbrack{[\W_L]^r, \v_i}\cdot \sigma'(\vbrack{\w_{r,L-1},\h_i})\cdot\h_i
    \end{displaymath}
    where \([\W_L]^r\) is the \(r\)-th column of \(\W_L\), and \(\w_{r,L-1} = \sqrt{m/2}[\W_{L-1}]_r\) is the \(r\)-th row of \(\sqrt{m/2}\W_{L-1}\). Obviously we have \(\sigma'(\vbrack{\w_{r,L-1},\h_i}) = \sigma'(\vbrack{[\W_{L-1}]_r,\h_i})\), and from our initialization scheme we also have \(\w_{r,L-1} \sim \N(0,\id_m)\). Now we define events \(\{A_i\}_{i\in [n]}\) over the randomness of \([\W_L]^r\) and \([\W_{L-1}]_r\) at initialization:
    \begin{displaymath}
        A_i = A_{i,1}\cap A_{i,2} \cap A_{i,3}
    \end{displaymath}
    where 
    \begin{itemize}
        \item \(A_{i,1}:= \{\w_{r,L-1} \in W_i\}\),
        \item \(A_{i,2}:= \{\|\h_r(\W_L,\W_{L-1})\|_2 \geq |\vbrack{[\W_L]^r, \v_i}|/4\}\),
        \item \(A_{i,3}:= \{  \vbrack{[\W_L]^r, \v_i}| \geq \norm{\v_i}_2/\sqrt{d} \}\).
    \end{itemize}
    Now by Lemma \ref{Lemma-C.1} and Lemma \ref{Lemma-C.2}, and the independence of \(\W_L\) and \(\W_{L-1}\), we have
    \begin{align*}
        &\P(r \in A_i) = \P(A_{i,2}|A_{i,1})\cdot \P(A_{i,1}) \cdot \P(A_{i,3}) \geq \frac{\delta}{256\sqrt{2}en} \quad \text{ and} \quad A_i \cap A_j = \varnothing \text{ if } i \neq j
    \end{align*}
    and also \(\sum_{i=1}^n \mathds{1}_{r \in A_i} \leq 1\). Therefore we can directly calculate
    \begin{align*}
        \sum_{r=1}^m\|\h_r(\W_L,\W_{L-1})\|_2^2 \geq  \sum_{r=1}^m\|\h_r(\W_L,\W_{L-1})\|_2^2\sum_{r=1}^m \mathds{1}_{r \in A_i} \geq  \sum_{r=1}^m\sum_{i=1}^n \frac{\norm{\v_i}_2^2}{32d}\mathds{1}_{r \in A_i} 
    \end{align*}
    Now define a random variable \(Z_r := \sum_{i=1}^n\mathds{1}_{r \in A_i}\norm{\v_i}_2^2/(32d)\), and from the definition of \(A_i\) we know that \((Z_r)_{r \in [m]}\) are independent (since \(\w_{r,L-1}, [\W_L]^r\) are independent for diffenrent \(r\)). Then for all \(r \in [m]\), we have 
    \begin{align*}
        \E[Z_r] \geq \Omega\biggl(\frac{\delta\sum_{i=1}^n\norm{\v_i}_2^2}{dn}\biggr),\quad (\E[Z_r])^2 \geq \Omega\biggl(\frac{\delta^2\sum_{i=1}^n\norm{\v_i}_2^4}{d^2n^2}\biggr) ,\quad \E[Z_r^2] \leq  O\biggl(\frac{\sum_{i=1}^n\norm{\v_i}_2^4}{d^2}\biggr)
    \end{align*}
    From one-sided Bernstein inequality for nonnegative random variables (see equation (2.23) in \citet{wainwright2019high}), we have
    \begin{align*}
        \P \biggl( \sum_{r=1}^m (Z_r - E[Z_r]) \leq \frac{m}{2}\biggl(\frac{1}{m}\sum_{r=1}^m\E[Z_r]\biggr) \biggr) &\leq \exp \biggl\{ - \Omega \biggl(\frac{m(\sum_{r=1}^m\E[Z_r]/m)^2}{\sum_{r=1}^m\E[Z^2_r]/m} \biggr)\biggr\}\\
        & \leq \exp \biggl\{ - \Omega \biggl(\frac{m^2\min_{r \in [m]}(\E[Z_r])^2}{\sum_{r=1}^m\E[Z^2_r]} \biggr)\biggr\} \\
        & \leq \exp(-\Omega(m\delta^2/n^2))
    \end{align*}
    which means, with probability at least \(1 - e^{-\Omega(m\delta^2/n^2)}\),
    \begin{align*}
        \frac{1}{n^2}\sum_{r=1}^m\|\h_r(\W_L,\W_{L-1})\|_2^2 \geq \frac{1}{2n^2}\sum_{r=1}^m\E[Z_r] \geq \Omega\biggl(\frac{m\delta}{n^3d}\sum_{i=1}^n\norm{\v_i}_2^2\biggr)
    \end{align*}
    Therefore we have proved the case of fixed vectors \((\v_i)_{i \in [n]}\). Applying \(\varepsilon\)-net argument, we know that for \(m \geq \Omega(n^3d\delta^{-2})\), the probability bound \(1 - e^{-\Omega(m\delta^2/n^2)}\) still holds. This concludes the proof.
\end{proof}

\subsection{Gradient Bounds at Initialization}\label{subsec:gradient-bound-init}

We first derive the gradient bounds for updating both \(\W\) and \(\theta\) at their random initializations, the result is summarized in the following lemma.
\begin{lemma}[Gradient Bounds at Initialization]\label{lem:gradient-bounds-init}
    With probability at least \(1 - 2e^{-\Omega(m\delta^2/n^2)} \), the following holds
    \begin{itemize}
        \item For \(\norm{\nabla_\W L_S(\W^{(0)},\theta^{(0)})}_F\), we have 
        \begin{align*}
            \Omega\biggl(\frac{m\delta}{n^3d}\biggr)\sum_{i=1}^n\|\losstilde_i\|_2^2 \leq \| \nabla_{\W}L_S(\W^{(0)},\theta^{(0)})\|_F^2 \leq O\biggl(\frac{Lm}{nd}\biggr)\sum_{i=1}^n\|\losstilde_i\|_2^2
        \end{align*}
        \item For \(\norm{\nabla_{\theta}L_S(\W^{(0)},\theta^{(0)}}_F\), we have
        \begin{align*}
            \Omega\biggl(\frac{m\delta}{n^3d}\biggr)\sum_{i=1}^n\|\losshat_i\|_2^2 \leq \| \nabla_{\theta}L_S(\W^{(0)},\theta^{(0)})\|_F^2 \leq O\biggl(\frac{Lm}{nd}\biggr)\sum_{i=1}^n\|\losshat_i\|_2^2
        \end{align*}
    \end{itemize}
\end{lemma}

\begin{proof}
    In the proof below, we drop all the superscripts appeared in \(\W^{(0)}\) and \(\theta^{(0)}\) for simplicity.

    \noindent\textbf{1. Gradient Upper Bound for updating the query encoder \(f^q_{\W}(\x_i)\):} For each \(i \in [n]\) and \(l \in [L]\), we calculate
    \begin{align*}
        \Bigl\| \nabla_{\W_l}L_S(\W,\theta)\Bigr\|_F &= \biggl\| \frac{1}{n}\sum_{i=1}^n \D_{i,l}\cdot\losstilde_i^{\top}(\back_{i,l}^q)\cdot\h_{i,l-1}^{\top}\biggr\|_F\\
        & \leq \frac{1}{n}\sum_{i=1}^n \norm{\D_{i,l}}_2\cdot \bigl\|\losstilde_i^{\top}(\back_{i,l}^q)\bigr\|_2\cdot \norm{\h_{i,l-1}}_2\\
        & \stackrel{\text{\ding{172}}}\leq O(\sqrt{m/d})\cdot\frac{1}{n}\sum_{i=1}^n \|\losstilde_i\|_2 \leq O(\sqrt{m/nd})\cdot \biggl(\sum_{i=1}^n \norm{\losstilde_i}_2^2\biggr)^{1/2}
    \end{align*}
    where the inequality \ding{172} has employed Lemma \ref{Lemma-7.1} and Lemma \ref{lem:backward-propagate} with probability at least \(1 - \exp(-\Omega(m/L))\). Taking squares and summing over \(l \in [L]\) give the desired result.

    \noindent\textbf{2. Gradient Lower Bound for updating the query encoder \(f^q_{\W}(\x_i)\):} Applying Lemma \ref{lem:gradient-lower-bound}, we have, with probability at least \(1- \exp(-\Omega(m\delta^2/n^2))\), the following lower bound holds:
    \begin{align*}
        \norm{\nabla_{\W}L_S(\W,\theta)}_F^2 &\geq \norm{\nabla_{\W_{L-1}}L_S(\W,\theta)}_F^2 \\
        &= \sum_{r=1}^m \biggl\| \frac{1}{n}\sum_{i=1}^n \vbrack{[\W_L]_r,\losstilde_i}\sigma'(\vbrack{[\W_{L-1}]_{r},\h_{i,L-2}})\h_{i,L-2}\biggr\|_2^2\\
        &\geq \Omega\biggl(\frac{m\delta}{n^3d}\biggr)\sum_{i=1}^n \norm{\losstilde_i}_2^2
    \end{align*}

    \noindent\textbf{3. Gradient Upper Bound for updating the key encoder \(f^k_{\theta}(\x_i)\):} From previous calculations \eqref{eq:gradient-for-theta}, we have
    \begin{align*}
        \norm{\nabla_{\theta_l} L_S(\W,\theta)}_F &= \biggl\| \frac{1}{n}\sum_{i=1}^n  \D_{i,l}^k\Bigl((\back_{i,l+1}^k)^{\top} \losshat_i\Bigr)\h_{i,l-1}^{\top}\biggr\|_F
    \end{align*}
    Therefore
    \begin{align*}
        \norm{\nabla_{\theta_l} L_S(\W,\theta)}_F &\leq \frac{1}{n}\sum_{i=1}^n \norm{\D_{i,l}}_2\cdot \norm{\losshat_i^{\top}\back_{i,l+1}^k}_2\cdot \norm{\h_{i,l-1}}_2\\
        & \stackrel{\text{\ding{172}}}\leq O(\sqrt{m/d})\cdot \frac{1}{n}\sum_{i=1}^n \norm{\losshat_i}_2 \leq O(\sqrt{m/nd})\biggl(\sum_{i=1}^n\norm{\losshat_i}_2^2\biggr)^{1/2}
    \end{align*}
    where in \ding{172} we have used Lemma \ref{Lemma-7.1}, Lemma \ref{lem:backward-propagate} again, with probability at least \(1 - \exp(-\Omega(m/L))\). Summing over \(l \in {L}\) gives the desired result.

    \noindent\textbf{4. Gradient Lower Bound for updating the key encoder \(f^k_{\theta}(\x_i)\):} From \eqref{eq:gradient-for-theta}, we can rewrite the Frobenius norm of the gradient \(\nabla_{\theta_L}L_S(\W,\theta_L)\) to the following form:
    \begin{displaymath}
        \norm{\nabla_{\theta_{L-1}} L_S(\W,\theta)}_F^2 = \sum_{r=1}^{m}\biggl\| \frac{1}{n}\sum_{i=1}^n  \vbrack{[\theta_L]_r, \losshat_i}\cdot \sigma'(\vbrack{[\theta_{L-1}]_r,\h_{i,L-2}})\cdot\h_{i,L-2}\biggr\|_2^2
    \end{displaymath}
    Applying Lemma \ref{lem:gradient-lower-bound}, we have, with probability at least \(1- \exp(-\Omega(m\delta^2/n^2))\), the following lower bound holds:
    \begin{displaymath}
        \norm{\nabla_{\theta} L_S(\W,\theta)}_F^2 \geq \norm{\nabla_{\theta_{L-1}} L_S(\W,\theta)}_F^2 \geq \Omega\biggl(\frac{m\delta}{n^3d}\biggr)\sum_{i=1}^n \norm{\losshat_i}_2^2
    \end{displaymath}
    Thus all the claims are proven.
\end{proof}

\subsection{Gradient Bounds After Pertubations}

Since we require the trajectory of the updated parameters \(\W^{(t)}\) and \(\theta^{(t)}\) to stay within certain neighborhoods \(B(\W^{(0)},\omega)\) and \(B(\theta^{(0)},\tau)\) of the random initilization, we need to prove that the gradient bounds remain valid in the neighborhood, which concludes of proof of Lemma \ref{lem:gradient-bounds}

\begin{proof}[Proof of Lemma \ref{lem:gradient-bounds}]
    Denote \(\D^{(0)}_{i,l}:= \diag\Bigl(\mathds{1}\{[\W^{(0)}]_r^{\top} \h^{(0)}_{i,l-1} \geq 0 \}_{r=1}^m\Bigr)\) and \(\h_{i,l}^{(0)} = \D^{(0)}_{i,l}\W^{(0)}_l \h_{i,l-1}^{(0)}\) to be the activated relus and the hidden-states of \(l\)-th layer for input \(\x_i\) at initialization, with \(\D_{i,l}, \h_{i,l}\) their perturbed counterparts. Also,for simplicity we define 
    \begin{displaymath}
        \back_{i,l}^{q,(0)} = \W^{(0)}_L \D^{(0)}_{i,L-1}\W^{(0)}_{L-1} \cdots \D^{(0)}_{i,l}\W^{(0)}_l 
    \end{displaymath}
    and \(\v_i = \losstilde_i\). The case of \(\W_L\) is trivial, for \(l \leq L-1\), we can calculate
    \begin{align*}
        &\nabla_{\W_l} L_S(\W^{(0)},\theta) - \nabla_{\W_l} L_S(\W,\theta)\\
        = \ &\frac{1}{n}\sum_{i=1}^n \biggl(  \v_i^{\top}(\back_{i,l+1}^{q,(0)}\D^{(0)}_{i,l})\cdot (\h_{i,l-1}^{(0)})^{\top} - \v_i^{\top}(\back_{i,l+1}^{q}\D_{i,l}) \cdot (\h_{i,l-1})^{\top}\biggr)
    \end{align*}
    From Lemma \ref{lem:backward-perturbation}, we have 
    \begin{displaymath}
        \bigl\|\v_i^{\top}\back_{i,l+1}^{q,(0)}\D^{(0)}_{i,l}-\v_i^{\top}\back_{i,l+1}^{q}\D_{i,l}\bigr\|_2 \leq O(\omega^{1/3}L^2\sqrt{m\log m/d})\cdot\norm{\v_i}_2
    \end{displaymath}
    From Lemma \ref{lem:backward-propagate}, we have 
    \begin{displaymath}
        \bigl\|\v_i^{\top}\back_{i,l+1}^{q,(0)}\D^{(0)}_{i,l}\bigr\|_2 \leq O(\sqrt{m/d})\cdot\norm{\v_i}_2
    \end{displaymath}
    By Lemma \ref{Lemma-7.1} and Lemma \ref{lem:output-perturbation}, we have, for all \(i \in [n]\)
    \begin{displaymath}
        \norm{\h^{(0)}_{i,l-1}}_2 \leq O(1)\text{  and  } \norm{\h_{i,l-1} - \h_{i,l-1}^{(0)}}_2 \leq O(\omega L^{3/2}) \implies \norm{\h_{i,l-1}}_2 \leq O(1)
    \end{displaymath}
    Putting together we arrive at 
    \begin{align*}
        \Bigl\| \nabla_{\W_l} L_S(\W^{(0)},\theta) - \nabla_{\W_l} L_S(\W,\theta) \Bigr\|_F^2 &\leq  \frac{1}{n}\sum_{i=1}^n\biggl\|  \v_i^{\top}(\back_{i,l+1}^{q,(0)}\D^{(0)}_{i,l})\cdot (\h_{i,l-1}^{(0)} - \h_{i,l-1})^{\top}\biggr\|_F^2\\
        & \quad + \frac{1}{n}\sum_{i=1}^n \biggl\| \v_i^{\top}\Bigl(\back_{i,l+1}^{q,(0)}\D^{(0)}_{i,l} - \back_{i,l+1}^{q}\D_{i,l}\Bigr)\cdot \h_{i,l-1}^{\top} \biggr\|_F^2\\
        & \leq O\biggl( \frac{\omega^{2/3}L^4m\log m}{nd} \biggr)\cdot \sum_{i=1}^n\norm{\v_i}_2^2\\
        & \stackrel{\text{\ding{172}}}\leq O\biggl(\frac{m\delta}{n^3d}\biggr)\sum_{i=1}^n\norm{\v_i}_2^2
    \end{align*}
    where \ding{172} is from our choice of \(\omega\). By  summing over \(l \in [L]\), we arrive at the desired results. Note that any change of \(\theta\) only affect the vector \((\v_i)_{i \in [n]}\), thus our analysis is still valid. The case of \(\|\nabla_{\theta}L_S(\W,\theta)\|_F\) can be similarly proved. 
\end{proof}

\section{The Semi-smoothness Property}

In this section, we prove Lemma \ref{lem:semi-smoothness}. Firstly, we present the following two lemmas and their proofs.

\subsection{Technical Lemmas}

\begin{lemma}\label{lem:L-smoothness-1}
    Let \( y = (y_{i})_{i \in [k]}\). The function \(g(y) = \log(1+\sum_{i=1}^k \exp(y_i)) \) is \(1\)-Lipschitz smooth with respect to \((y_{i})_{i \in [k]}\) and satisfies 
    \begin{displaymath}
        g(y+y') \leq g(y) + \nabla_y g(y)^{\top}y' + \frac{1}{2}\norm{y'}_2^2
    \end{displaymath}
\end{lemma}

\begin{proof}
    Trivially this function (cross-entropy loss) is convex with respect to \(y = (y_{i})_{i=1}^k\), which means the Hessian \(\nabla^2g(y)\) is positive-semidefinite. And we can calculate 
    \begin{displaymath}
        (\nabla^2g(y))_{i,i} = \frac{\exp(y_{i})}{(1+\sum_{s=1}^k\exp(y_{s}))^2}\Bigl(1 + \sum_{j \neq i}\exp(y_{j}) \Bigr)
    \end{displaymath}
    Summing over \(i \in [k]\), we have \(\sum_{i=1}^k (\nabla^2g(y))_{i,i} \leq 1\). And since \(g(y)\) is convex, the eigenvalues \((\lambda_i)_{i \in [k]}\) of \(\nabla^2g(y)\) satisfies \(\lambda_i \geq 0\) and \(\norm{\nabla^2g(y)}_2 \leq \sum_{i=1}^k\lambda_i = \sum_{i=1}^k (\nabla^2g(y))_{i,i} \leq 1\). Note that the bound for \(\norm{\nabla^2g(y)}_2\) is valid for all \(y = (y_{i}) \in \R^k\), which proves the claim by doing simple Taylor expansion.
\end{proof}

\begin{lemma}\label{lem:L-smoothness-2}
    For \(\W, \widetilde{\W} \in B(\W^{(0)},\omega)\) and \(\theta,\widetilde{\theta} \in B(\theta^{(0)},\tau)\), where 
    \begin{displaymath}
        \omega, \tau\in [\Omega(\sqrt{d/m}),\ O(1/ (L^{9/2} (\log m)^{3/2}) )]
    \end{displaymath}
    we have, with probability at least \(1-e^{-\Omega(m\omega^{3/2}L)}-e^{-\Omega(m\tau^{3/2}L)}\) over the initialization,
    \begin{align*}
        L_S(\widetilde{\W},\widetilde{\theta}) - L_S(\W,\theta) &\leq \frac{1}{n}\sum_{i=1}^n \vbrack{\losstilde_i, \qt_i - q_i} + \vbrack{\losshat_i, \kt_i - k_i} \\
        & \quad + O\biggl(\frac{kL^2m^2}{d^2}\biggr)\Bigl( \tau^2\norm{\widetilde{\W} - \W}_2^2 + \omega^2\norm{\widetilde{\theta} - \theta}_2^2 \Bigr)
    \end{align*}
\end{lemma}

\begin{proof}
    Recall from Definition \ref{def:loss-function} that our loss function is of the form:
    \begin{align*}
        L_S(\W,\theta) &= \frac{1}{n}\sum_{i=1}^n \E^{Neg(i)} [\ell(f^q_{\W},f^k_{\theta},\x_i,\{\x_{i,j}\}_{j=1}^k)]\\
        & = \frac{1}{n}\sum_{i=1}^n\frac{1}{\binom{n-1}{k}}\sum_{\{\x_{i,j}\}_{j=1}^k \subset  S^{\setminus i}} \log \biggl(1 + \sum_{j=1}^k \exp(q_i^{\top}z_{i,j}) \biggr)\\
        & = \frac{1}{n}\sum_{i=1}^n \sum_{j\neq i} \sum_{ \x_{j} \in \{\x_{i,j}\}_{j=1}^k \subset  S^{\setminus i}} \frac{1}{\binom{n-1}{k}}\log \biggl(1 + \sum_{j=1}^k \exp(q_i^{\top}z_{i,j}) \biggr) 
    \end{align*}
    where \(q_i := f^q_{\W}(\x_i)\), \(z_{i,j}:= f^k_{\theta}(\x_{i,j}) - f^k_{\theta}(\x_i) \) and \(z_{j} = f^k_{\theta}(\x_{j}) - f^k_{\theta}(\x_i) \). Now for a set of different parameters \(\widetilde{\W} \in B(\W^{(0)},\omega)\) and \(\widetilde{\theta} \in B(\theta^{(0)},\tau)\), we define new queries and keys as 
    \begin{displaymath}
        \tilde{q}_i :=  f^q_{\widetilde{\W}}(\x_i),\qquad \tilde{z}_{i,j}:= f^k_{\widetilde{\theta}}(\x_{i,j}) - f^k_{\widetilde{\theta}}(\x_i)
    \end{displaymath}
    Applying Lemma \ref{lem:L-smoothness-1}, we have 
    \begin{align*}
        \ell(f^q_{\widetilde{\W}},f^k_{\widetilde{\theta}},\x_i,\{\x_{i,j}\}_{j=1}^k) &=\log \biggl(1 + \sum_{j=1}^k \exp(\qt_i^{\top}\zt_{i,j}) \biggr) \\
        &\leq \log \biggl(1 + \sum_{j=1}^k \exp(q_i^{\top}z_{i,j}) \biggr) \\
        &\quad + \sum_{j = 1}^k \frac{\exp(q_i^{\top}z_{i,j} )}{1 + \sum_{s=1}^k \exp(q_i^{\top}z_{i,s})}\cdot \Bigl( \qt_i^{\top}\zt_{i,j} - q_i^{\top}z_{i,j} \Bigr) \\
        &\quad + \frac{1}{2} \sum_{j=1}^k(\qt_i^{\top}\zt_{i,j} - q_i^{\top}z_{i,j})^2\\
        & = \ell(f^q_{\W},f^k_{\theta},\x_i,\{\x_{i,j}\}_{j=1}^k) + \Phi_1 + \Phi_2
    \end{align*}
    We now decompose \(\qt_i^{\top}\zt_{i,j} - q_i^{\top}z_{i,j}\):
    \begin{equation}\label{eq:decomposition}
        \qt_i^{\top}\zt_{i,j} - q_i^{\top}z_{i,j} = (\qt_i - q_i)^{\top}z_{i,j} +  q_i^{\top}(\zt_{i,j} - z_{i,j})+ (\qt_i - q_i)^{\top}(\zt_{i,j} - z_{i,j})
    \end{equation}
    Therefore \(\Phi_1\) can be calculated as 
    \begin{align*}
        \Phi_1 &=  \sum_{j = 1}^k \frac{\exp(q_i^{\top}z_{i,j} )}{1 + \sum_{s=1}^k \exp(q_i^{\top}z_{i,s})}\cdot \Bigl( \qt_i^{\top}\zt_{i,j} - q_i^{\top}z_{i,j} \Bigr)\\
        & =  \sum_{j = 1}^k \frac{\exp(q_i^{\top}z_{i,j} )}{1 + \sum_{s=1}^k \exp(q_i^{\top}z_{i,s})}\cdot z_{i,j}^{\top} (\qt_i - q_i)\\
        &\quad +  \sum_{j = 1}^k \frac{\exp(q_i^{\top}z_{i,j} )}{1 + \sum_{s=1}^k \exp(q_i^{\top}z_{i,s})}\cdot  q_i^{\top}  (\zt_{i,j} - z_{i,j})\\
        & \quad + \sum_{j = 1}^k \frac{\exp(q_i^{\top}z_{i,j} )}{1 + \sum_{s=1}^k \exp(q_i^{\top}z_{i,s})}\cdot (\qt_i - q_i)^{\top}(\zt_{i,j} - z_{i,j})\\
        & = \Psi_1(\x_i,\{\x_{i,j}\}_{j=1}^k) + \Psi_2(\x_i,\{\x_{i,j}\}_{j=1}^k) + \Psi_3(\x_i,\{\x_{i,j}\}_{j=1}^k)
    \end{align*}
    For \(\Psi_1(\x_i,\{\x_{i,j}\}_{j=1}^k)\), its expectation with respect to negative sampling \(\{\x_{i,j}\}_{j=1}^k \subset S^{\setminus i}\) is 
    \begin{align*}
        \E^{Neg(i)}[\Psi_1(\x_i,\{\x_{i,j}\}_{j=1}^k)] & = \E^{Neg(i)} \biggl[\frac{\exp(q_i^{\top}z_{i,j} )}{1 + \sum_{s=1}^k \exp(q_i^{\top}z_{i,s})}\cdot z_{i,j}\biggr]^{\top} (\qt_i - q_i)\\
        & \stackrel{\text{\ding{172}}}= \langle \losstilde_i, (\qt_i - q_i)\rangle
    \end{align*}
    where \ding{172} is from Definition \ref{def:loss-vector}, which implies 
    \begin{equation}\label{eq:Phi-1}
        \frac{1}{n}\sum_{i=1}^n \E^{Neg(i)}[\Psi_1(\x_i,\{\x_{i,j}\}_{j=1}^k)] = \frac{1}{n} \sum_{i=1}^n \vbrack{\losstilde_i , \qt_i - q_i}
    \end{equation}
    Now for \(\Psi_2(\x_i,\{\x_{i,j}\}_{j=1}^k)\), we calculate
    \begin{align}\label{eq:Psi-2-expectation}
        & \frac{1}{n} \sum_{i=1}^n\E^{Neg(i)}[\Psi_2(\x_i,\{\x_{i,j}\}_{j=1}^k)] \nonumber\\
        = \ & \frac{1}{n}\sum_{i=1}^n \frac{1}{\binom{n-1}{k}}\sum_{  \{\x_{i,j}\}_{j=1}^k \subset  S^{\setminus i}} \sum_{j = 1}^k \frac{\exp(q_i^{\top}z_{i,j} )}{1 + \sum_{s=1}^k \exp(q_i^{\top}z_{i,s})}\cdot  q_i^{\top}  (\zt_{i,j} - z_{i,j}) \nonumber\\
        = \ & \frac{1}{n}\sum_{i=1}^n \frac{1}{\binom{n-1}{k}}\sum_{ \{\x_{i,j}\}_{j=1}^k \subset  S^{\setminus i}} \sum_{j = 1}^k \frac{\exp(q_i^{\top}z_{i,j} )}{1 + \sum_{s=1}^k \exp(q_i^{\top}z_{i,s})}\cdot  q_i^{\top}  (\zt_{i,j} - z_{i,j})
    \end{align}
    Now set notations
    \begin{displaymath}
        k_{i,j} = f^k_{\theta}(\x_{i,j}),\quad \tilde{k}_{i,j} = f^k_{\widetilde{\theta}}(\x_{i,j}),\quad k_{i} = f^k_{\theta}(\x_{i}),\quad \tilde{k}_{i} = f^k_{\widetilde{\theta}}(\x_{i})
    \end{displaymath}
    then we can rearrange \eqref{eq:Psi-2-expectation} to 
    \begin{align*}
        &\frac{1}{n} \sum_{i=1}^n\E^{Neg(i)}[\Psi_2(\x_i,\{\x_{i,j}\}_{j=1}^k)] \\
        = \ & \frac{1}{n}\sum_{i=1}^n \sum_{j \neq i} \frac{1}{\binom{n-1}{k}}\sum_{ \x_{j} \in \{\x_{i,j}\}_{j=1}^k \subset  S^{\setminus i}} \frac{\exp(q_i^{\top}(k_j - k_i) )}{1 + \sum_{s=1}^k \exp(q_i^{\top}z_{i,s})}\cdot  q_i^{\top} \Bigl((\kt_j - k_j) - (\kt_i - k_i) \Bigr)\\
        \stackrel{\text{\ding{172}}}= \ & \frac{1}{n}\sum_{i=1}^n \sum_{j \neq i}\losshat(\x_i,\x_j)^{\top} \Bigl((\kt_j - k_j) - (\kt_i - k_i) \Bigr)\\
        = \ & \frac{1}{n} \sum_{i=1}^n \sum_{j\neq i} (\losshat(\x_j,\x_i) - \losshat(\x_i,\x_j))^{\top} (\kt_i - k_i)\\
        \stackrel{\text{\ding{173}}}= \ & \frac{1}{n}\sum_{i=1}^n \vbrack{\losshat_i, \kt_i - k_i}
    \end{align*}
    where \ding{172} and \ding{173} are both from Definition \ref{def:loss-vector}. For \(\Psi_3(\x_i,\{\x_{i,j}\}_{j=1}^k)\), we can use Cauchy-Schwarz inequality to get 
    \begin{align*}
        \Psi_3(\x_i,\{\x_{i,j}\}_{j=1}^k) &\leq \sum_{j = 1}^k \frac{\exp(q_i^{\top}z_{i,j} )}{1 + \sum_{s=1}^k \exp(q_i^{\top}z_{i,s})}\cdot \|\qt_i - q_i\|_2\cdot\|\zt_{i,j} - z_{i,j}\|_2\\
        & \stackrel{\text{\ding{172}}}\leq O\biggl(\frac{L^2m}{d}\biggr)\cdot \norm{\widetilde{\W} - \W}_2\cdot\norm{\widetilde{\theta} - \theta}_2\\
        & \leq O\biggl(\frac{L^2m}{d}\biggr)\Bigl( \norm{\widetilde{\W} - \W}_2^2 + \norm{\widetilde{\theta} - \theta}_2^2\Bigr)
    \end{align*}
    where in \ding{172} we have employed Lemma \ref{lem:output-perturbation}, which requires \(\norm{\widetilde{\W} - \W}_2 \leq \omega\) and \(\norm{\widetilde{\theta} - \theta}_2 \leq \tau\). This implies
    \begin{displaymath}
        \frac{1}{n}\sum_{i=1}^n\E^{Neg(i)}[\Psi_3(\x_i,\{\x_{i,j}\}_{j=1}^k)] \leq O\biggl(\frac{L^2m}{d}\biggr)\Bigl( \norm{\widetilde{\W} - \W}_2^2 + \norm{\widetilde{\theta} - \theta}_2^2\Bigr)
    \end{displaymath}
    Now we come to deal with \(\Phi_2\). From the decomposition \eqref{eq:decomposition} we have 
    \begin{align*}
        \Phi_2 &= \frac{1}{2} \sum_{j=1}^k(\qt_i^{\top}\zt_{i,j} - q_i^{\top}z_{i,j})^2 \\
        & \leq \frac{3}{2} \sum_{j=1}^k \Bigl( \norm{\qt_i - q_i}_2^2\cdot \norm{z_{i,j}}_2^2 + \norm{q_i}_2^2\cdot\norm{\zt_{i,j} - z_{i,j}}_2^2 + \norm{\qt_i - q_i}_2^2\cdot\norm{\zt_{i,j} - z_{i,j}}_2^2 \Bigr)\\
        &\stackrel{\text{\ding{172}}}\leq O\biggl(\frac{kL^2m^2}{d^2}\biggr)\cdot \Bigl( \tau^2\norm{\widetilde{\W} - \W}_2^2 + \omega^2\norm{\widetilde{\theta} - \theta}_2^2 + \tau\norm{\widetilde{\W} - \W}_2 \cdot \omega\norm{\widetilde{\theta} - \theta}_2^2 \Bigr)\\
        & \stackrel{\text{\ding{173}}}\leq  O\biggl(\frac{kL^2m^2}{d^2}\biggr)\cdot \Bigl( \tau^2\norm{\widetilde{\W} - \W}_2^2 + \omega^2\norm{\widetilde{\theta} - \theta}_2^2 \Bigr)
    \end{align*}
    where \ding{172} have employed Lemma \ref{lem:output-boundedness} to obtain \(\norm{q_i}_2^2, \norm{z_{i,j}}_2^2 \leq O(1)\) (smaller than \(\omega L\sqrt{m/d}\)) at initialization and Lemma \ref{lem:output-perturbation} to obtain 
    \begin{displaymath}
        \norm{\qt_i - q_i}_2^2 \leq O\biggl(\frac{L^2m}{d}\biggr)\cdot\norm{\widetilde{\W} - \W}_2^2,\qquad \norm{\zt_{i,j} - z_{i,j}}_2^2 \leq O\biggl(\frac{L^2m}{d}\biggr)\cdot\norm{\widetilde{\theta} - \theta}_2^2
    \end{displaymath}
    and \ding{173} is due to Cauchy-Schwarz inequality. Thus we can prove the claim by taking expectations with respect to negative sampling of \(\{\x_{i,j}\}_{j=1}^k \subset S^{\setminus i}\) and sum over \(i \in [n]\).
\end{proof}

\subsection{Proof of Lemma \ref{lem:semi-smoothness}}

\begin{proof}[Proof of Lemma \ref{lem:semi-smoothness}]
    Similar to the proofs of previous lemmas, we set notations as follows. For parameters \(\W',\theta'\) with \(\norm{\W'}_2 \leq \omega,\ \norm{\theta'}_2 \leq \tau\), we denote
    \begin{displaymath}
        \qt_i = f^q_{\W+\W'}(\x_i), \qquad q_i = f^q_{\W}(\x_i), \qquad \kt_i = f^k_{\theta+\theta'}(\x_i), \qquad k_i = f^k_{\theta}(\x_i).
    \end{displaymath}
    Applying Lemma \ref{lem:L-smoothness-2}, we can calculate 
    \begin{align*}
        &L_S(\W + \W',\theta + \theta') - L_S(\W,\theta) - \vbrack{\nabla_{\W}L_S(\W,\theta),\W'} - \vbrack{\nabla_{\theta}L_S(\W,\theta),\theta'} \\
        \leq \ &- \vbrack{\nabla_{\W}L_S(\W,\theta),\W'} - \vbrack{\nabla_{\theta}L_S(\W,\theta),\theta'} + \frac{1}{n}\sum_{i=1}^n \vbrack{\losstilde_i, \qt_i - q_i} + \vbrack{\losshat_i, \kt_i - k_i}  \\
        & + O\biggl(\frac{kL^2m^2}{d^2}\biggr)\Bigl( \tau^2\norm{\W'}_2^2 + \omega^2\norm{\theta'}_2^2 \Bigr)\\
        = \ & F_1 + F_2 + F_3
    \end{align*}
    where 
    \begin{align*}
        F_1 & = - \vbrack{\nabla_{\W}L_S(\W,\theta),\W'} + \frac{1}{n}\sum_{i=1}^n \vbrack{\losstilde_i, \qt_i - q_i}\\
        F_2 & = - \vbrack{\nabla_{\theta}L_S(\W,\theta),\theta'} + \frac{1}{n}\sum_{i=1}^n\vbrack{\losshat_i, \kt_i - k_i}\\
        F_3 & = O\biggl(\frac{kL^2m^2}{d^2}\biggr)\Bigl( \tau^2\norm{\W'}_2^2 + \omega^2\norm{\theta'}_2^2 \Bigr)
    \end{align*}
    The goal here is to obtain bounds for \(F_1\) and \(F_2\), so we divide our proof into two steps:

    \noindent \textbf{Step 1. The case of \(F_1\):}\\
    For \(F_1\) we have \(F_1 = \frac{1}{n}\sum_{i=1}^n F_1^i\), where \(F_1^i\) can be calculated as
    \begin{align*}
        F_1^i = \ & \losstilde_i^{\top} \Biggl(f^q_{\W + \W'}(\x_i) - f^q_{\W}(\x_i) - \sum_{l=0}^L \W_L\biggl(\prod_{a=l+1}^{L-1} \D_{i,a}\W_{a}\biggr)\D_{i,l}\W'_{l}\h_{i,l-1} \Biggr)\\
        = \ & \losstilde_i^{\top} \Biggl((\W_L+\W'_L)\h'_{i,L-1} - \W_L\h_{i,L-1} - \sum_{l=0}^L \W_L\biggl(\prod_{a=l+1}^{L-1} \D_{i,a}\W_{a}\biggr)\D_{i,l}\W'_{l}\h_{i,l-1} \Biggr)
    \end{align*}
    Now recall our notations:
    \begin{displaymath}
        \h'_{i,-1} = \x_i,\qquad \h'_{i,l} = \sigma((\W_l + \W'_l)\h'_{i,l-1})\quad \text{for } 0 \leq l \leq L-1
    \end{displaymath}
    By applying Lemma \ref{lem:output-perturbation}, for all \(i\in [n]\) and \(k \in [m]\), there exist diagonal matrices \(\D''_{i,l}\) such that \(|(\D_{i,l}+\D''_{i,l})_{k,k}| \leq 1\), and 
    \begin{align*}
        &(\W_L+\W'_L)\h'_{i,L-1} - \W_L\h_{i,L-1} \\
        = \ & \W'_L\h'_{i,L-1} + \sum_{l=0}^{L-1} \W_L\biggl(\prod_{a=l+1}^{L-1} (\D_{i,a} + \D''_{i,a})\W_{a}\biggr)(\D_{i,l}+\D''_{i,l})\W'_{l}\h'_{i,l-1}
    \end{align*}
    So we can further calculte 
    \begin{align}
        F_1^i &= \losstilde_i^{\top}   \biggl[ \W'_L\h'_{i,L-1} + \sum_{l=0}^{L-1}\W_L\biggl(\prod_{a=l+1}^{L-1} (\D_{i,a} + \D''_{i,a})\W_{a}\biggr)(\D_{i,l}+\D''_{i,l})\W'_{l}\h'_{i,l-1} \nonumber\\
        &\qquad \qquad \qquad - \W'_L\h_{i,L-1} +\sum_{l=0}^{L-1}\W_L\biggl(\prod_{a=l+1}^{L-1} \D_{i,a}\W_{a}\biggr)\D_{i,l}\W'_{l}\h_{i,l-1}\biggr] \nonumber\\
        & = \losstilde_i^\top \sum_{l=0}^{L-1}\biggl[ \W_L\prod_{a=l+1}^{L-1} \bigl((\D_{i,a} + \D''_{i,a})\W_{a}\bigr)(\D_{i,l}+\D''_{i,l}) - \W_L\prod_{a=l+1}^{L-1} \bigl(\D_{i,a}\W_{a}\bigr)\D_{i,l}\biggr]\W'_{l}\h'_{i,l-1}\nonumber \\
        & \qquad + \losstilde_i^{\top} \biggl[\W'_L + \sum_{l=0}^{L-1} \W_L\biggl(\prod_{a=l+1}^L \D_{i,a}\W_{a}\biggr)\D_{i,l}\W'_{l} \biggr](\h'_{i,l-1} -\h_{i,l-1})\nonumber \\
        & =  \sum_{l=0}^{L-1}Q_1^l + \sum_{l=0}^LQ_2^l\label{eq:semi-1-decom-2}
    \end{align}
    Therefore, we can bound the two terms \(Q_1^l\) and \(Q_2^l\) separately. For \(Q_1^l\), we apply Lemma \ref{lem:backward-perturbation} with \(s = O(m\omega^{2/3}L)\) (where the choice of \(s\) is from Lemma \ref{Lemma-8.2}(b)), the Cauchy-Schwarz theorem, and the boundedness of \(\h'_{i,l-1}\) with respect to perturbations to get:
    \begin{align}
        Q_1^l & = \losstilde_i^\top \biggl[ \W_L\biggl(\prod_{a=l+1}^L (\D_{i,a} + \D''_{i,a})\W_{a}\biggr)(\D_{i,l}+\D''_{i,l})- \W_L\biggl(\prod_{a=l+1}^L \D_{i,a}\W_{a}\biggr)\D_{i,l}\biggr]\W'_{l}\h'_{i,l-1} &\nonumber\\
        & \leq \norm{\losstilde_i}\Biggl\| \W_L\prod_{a=l+1}^L \bigl((\D_{i,a} + \D''_{i,a})\W_{a}\bigr)(\D_{i,l}+\D''_{i,l}) - \W_L\prod_{a=l+1}^L (\D_{i,a}\W_{a})\D_{i,l}\Biggr\| \norm{\W'_{l}\h'_{i,l-1}} \nonumber \\
        & \leq \norm{\losstilde_i}_2 \cdot O\biggl(\frac{\omega^{1/3}L^2 \sqrt{m\log m}}{\sqrt{d}}\biggr) \cdot \norm{\W'_l}_F \label{eq:Q-1-bound}
    \end{align}
    and for the second term \(Q_2^l\), when \(l = L\), we have 
    \begin{displaymath}
        Q_2^l \leq \norm{\losstilde_i}_2\cdot\norm{\W'_L}_2\cdot\norm{\h'_{i,L-1} -\h_{i,L-1}}_2 \leq \norm{\losstilde_i}_2\cdot O\biggl(\frac{\omega^{1/3}L^2\sqrt{\log m}}{\sqrt{d}}\biggr)\cdot\norm{\W'}_F
    \end{displaymath}
    for \(l\leq L-1\), we calculate
    \begin{align*}
        Q_2^l & = \losstilde_i^{\top} \biggl[\W_L\biggl(\prod_{a=l+1}^L \D_{i,a}\W_{a}\biggr)\D_{i,l}\W'_{l} \biggr](\h'_{i,l-1} -\h_{i,l-1})\\
        & =  \losstilde_i^{\top}\biggl[\W_L\biggl(\prod_{a=l+1}^{L-1} \D_{i,a}\W_{a}\biggr)\D_{i,l}-\B\biggl(\prod_{a=l+1}^L \D^{0}_{i,a}\W^{(0)}_{a}\biggr)\D^{(0)}_{i,l} \biggr]\W'_{l} (\h'_{i,l-1} -\h_{i,l-1})\\
        &\qquad +  \losstilde_i^{\top} \W_L\biggl(\prod_{a=l+1}^{L-1} \D^{(0)}_{i,a}\W^{(0)}_{a}\biggr)\D^{(0)}_{i,l}\W'_{l} (\h'_{i,l-1} -\h_{i,l-1})\\
        & \leq \norm{\losstilde_i}_2 \cdot O\biggl(\sqrt{\frac{m}{d}}  + \frac{\omega^{1/3}L^2 \sqrt{m\log m}}{\sqrt{d}}\biggr)\cdot\norm{\W'_l}_F\cdot\norm{\h'_{i,l-1} -\h_{i,l-1}}_2 
    \end{align*}
    Again from Lemma \ref{lem:output-perturbation} that \(\norm{\h'_{i,l} - \h_{i,l}} \leq O(L^{3/2})\norm{\W}_2\) and our choice of \(\omega\) we have 
    \begin{equation}
        Q_2^l \leq \norm{\losstilde_i}_2 \cdot O\biggl(\frac{\omega^{1/3}L^2 \sqrt{m\log m}}{\sqrt{d}}\biggr)\cdot \norm{\W'_l}_F \label{eq:Q-2-bound}
    \end{equation}
    Combining \eqref{eq:Q-1-bound} and \eqref{eq:Q-2-bound}, we have 
    \begin{equation}\label{eq:F-1-bound}
        \begin{split}
        F_1 = \frac{1}{n}\sum_{i=1}^n F_1^i &\leq \frac{1}{n}\sum_{i=1}^n \sum_{l =0}^L\norm{\losstilde_i}_2\cdot O\biggl(\frac{\omega^{1/3}L^2 \sqrt{m\log m}}{\sqrt{d}}\biggr)\cdot \norm{\W'_l}_F \\
        & \leq \norm{\losstilde}_2\cdot O\biggl(\frac{\omega^{1/3}L^2 \sqrt{m\log m}}{\sqrt{nd}}\biggr)\cdot \norm{\W'}_F
        \end{split}
    \end{equation}
    which proves the case of \(F_1\).

    \noindent \textbf{Step 2. The case of \(F_2\):}\\
    we rearrange the first order term \(\vbrack{\nabla_{\theta} L_S(\W,\theta),\theta'}\) to a more operable form: first we calculate, as in \hyperref[subsec:key-calculations]{Section \ref*{subsec:key-calculations}},
    \begin{align} \label{eq:first-order-term}
        &\vbrack{\nabla_{\theta} L_S(\W,\theta),\theta'} \nonumber\\
        = \ & \biggl\langle \frac{1}{n}\sum_{i=1}^n \E^{Neg} \biggl[\nabla_{\theta}\ell(\W,\theta,\x_i,\{\x_{i,j}\}_{j=1}^k)\biggr] , \theta'\biggr\rangle \nonumber \\
        = \ & \frac{1}{n}\sum_{i=1}^n \frac{1}{\binom{n-1}{k}} \sum_{\{\x_{i,j}\}_{j=1}^k \subset S^{\setminus i}} \sum_{j=1}^k \frac{\exp(q_i^{\top}z_{i,j})}{1 + \sum_{s=1}^k \exp(q_i^{\top} z_{i,s})} \bigl\langle\nabla_{\theta}\, q_i^{\top}(f^k_{\theta}(\x_{i,j}) - f^k_{\theta}(\x_{i})), \theta' \bigr\rangle \nonumber\\
        = \ & \frac{1}{n}\sum_{i=1}^n \sum_{j\neq i}  \losshat(\x_i,\x_j)^{\top}  \Bigl(\nabla_{\theta} f^k_{\theta}(\x_{j})(\theta') - \nabla_{\theta}f^k_{\theta}(\x_{i}) (\theta')\Bigr) \nonumber\\
        = \ & \frac{1}{n}\sum_{i=1}^n \sum_{j \neq i} \bigl(\losshat(\x_j,\x_i)- \losshat(\x_i,\x_j)\bigr)^{\top}  \nabla_{\theta} f^k_{\theta}(\x_{i}) (\theta')  \nonumber\\
        = \ & \frac{1}{n}\sum_{i=1}^n \losshat_i^{\top} \nabla_{\theta} f^k_{\theta}(\x_{i}) (\theta') 
    \end{align}
    where we have denote 
    \begin{displaymath}
        \nabla_{\theta}f^k_{\theta}(\x_i) (\theta') := \sum_{l=0}^L \theta_L\D_{i,L-1}\theta_{L-1}\cdots\theta_{l+1}\D_{i,l}\,\theta'_{l}\,\h_{i,l-1}
    \end{displaymath}
    Now we let \(F_2 = \frac{1}{n}\sum_{i=1}^n F_2^i\), where
    \begin{equation}\label{eq:F-2-i}
        F_2^i = \losshat_i^{\top} \biggl[ (\theta_L + \theta'_L) \h'_{i,L} - \theta_L\h_{i,L} - \sum_{l=0}^L\theta_L\biggl( \prod_{a=l+1}^{L-1} \D_{i,L}\theta_a\biggr)\D_{i,l}\,\theta'_{l}\,\h_{i,l-1} \biggr]
    \end{equation}
    by Lemma \ref{lem:output-perturbation}, we have
    \begin{align*}
        (\theta_L + \theta'_L) \h'_{i,L-1} - \theta_L\h_{i,L-1} &= \theta'_L\h'_{i,L-1} + \sum_{l=0}^{L-1} \biggl(\prod_{a=l+1}^{L-1} (\D_{i,a} + \D''_{i,a})\theta_{a}\biggr)(\D_{i,l}+\D''_{i,l})\,\theta'_{l}\,\h'_{i,l-1}
    \end{align*}
    Substitute this into equation \eqref{eq:F-2-i}, we can further calculate 
    \begin{align*}
        F_2^i &=\losshat_i^{\top}\biggl[ \theta'_L\h'_{i,L-1} + \sum_{l=1}^{L-1}\theta_L\biggl(\prod_{a=l+1}^{L-1} (\D_{i,a} + \D''_{i,a})\theta_{a}\biggr)(\D_{i,l}+\D''_{i,l})\theta'_{l}\,\h'_{i,l-1}\biggr] \\
        & \qquad - \losshat_i^{\top}\sum_{l=0}^L\theta_L\biggl( \prod_{a=l+1}^{L-1} \D_{i,a}\theta_a\biggr)\D_{i,l}\,\theta'_{l}\,\h_{i,l-1}\\
        & = \sum_{l=0}^{L-1}\losshat_i^{\top} \Biggl[ \theta_L\biggl(\prod_{a=l+1}^{L-1} (\D_{i,a} + \D''_{i,a})\theta_{a}\biggr)(\D_{i,l}+\D''_{i,l})\, \theta'_{l}\, - \theta_L\biggl(\prod_{a=l+1}^{L-1} \D_{i,a}\theta_{a}\biggr)\D_{i,l}\, \theta'_{l}\, \Biggr]\h'_{i,l-1} \\
        & \quad + \sum_{l=0}^L \losshat_i^{\top}\theta_L\biggl(\prod_{a=l+1}^L \D_{i,a}\theta_{a}\biggr)\D_{i,l}\, \theta'_{l}\, \bigl(\h'_{i,l-1} -\h_{i,l-1}\bigr)\\
        & = F^i_{2,1} + F^i_{2,2}
    \end{align*}
    Now we apply Cauchy-Schwarz inequality to \(F^i_{2,1}\) and get
    \begin{align*}
        F^i_{2,1} \stackrel{\text{\ding{172}}}\leq \ &  \sum_{l=0}^{L-1} \|\losshat_i\|_2 \biggl\|\theta_L\biggl(\prod_{a=l+1}^L (\D_{i,a} + \D''_{i,a})\theta_{a}\biggr)(\D_{i,l}+\D''_{i,l}) - \theta_L\biggl(\prod_{a=l+1}^L \D_{i,a}\theta_{a}\biggr)\D_{i,l} \biggr\|_2\norm{\theta'_{l}}_F\\
        \stackrel{\text{\ding{173}}}\leq \ & \norm{\losshat_i}\cdot O\biggl(\frac{\omega^{1/3}L^2\sqrt{m\log m}}{\sqrt{d}}\biggr)\cdot \norm{\theta'}_F
    \end{align*}
    where in \ding{172} we have used Lemma \ref{Lemma-7.1} and Lemma \ref{lem:output-boundedness} to obtain the boundedness of \(\norm{\h'_{i,l-1}}_2\), and in \ding{173} we have used Lemma \ref{lem:backward-perturbation}. On the other hand, we have
    \begin{align*}
        F^i_{2,2} &\leq \sum_{l=0}^{L-1} \|\losshat_i\|_2\cdot O\biggl(\sqrt{\frac{m}{d}}+ \frac{\tau^{1/3}L^2\sqrt{m\log m}}{\sqrt{d}}\biggr)\cdot \norm{\theta'_l}_F\cdot\norm{\h'_{i,l-1} -\h_{i,l-1}}_2\\
        & \leq \|\losshat_i\|_2\cdot O\biggl(\frac{\tau^{1/3}L^2\sqrt{m\log m}}{\sqrt{d}}\biggr)\cdot \norm{\theta'}_F
    \end{align*}
    via applying Lemma \ref{Lemma-7.1}, Lemma \ref{lem:backward-perturbation} and Lemma \ref{lem:output-perturbation}, and also by our choice of \(\tau\). This implies 
    \begin{align*}
        F_2 = \frac{1}{n}\sum_{i=1}^n F_2^i =\frac{1}{n}\sum_{i=1}^n (F^i_{2,1}+F^i_{2,2}) \leq  \|\losshat\|_2\cdot O\biggl(\frac{\tau^{1/3}L^2\sqrt{m\log m}}{\sqrt{nd}}\biggr)\cdot \norm{\theta'}_F
    \end{align*}
    And thus we conclude the proof.
\end{proof}

\end{document}